\documentclass[twoside]{article}

\usepackage[accepted]{aistats2025}
\usepackage[round]{natbib}

\usepackage{amsmath, amssymb, amsthm}
\usepackage{bm, dsfont}
\usepackage{mathtools}
\usepackage{hyperref}
\usepackage{xcolor}
\usepackage{subcaption}
\usepackage[group-minimum-digits = 4, group-separator = {,}]{siunitx}
\usepackage[ruled]{algorithm2e}

\mathtoolsset{showonlyrefs, showmanualtags}

\usepackage[columnwise, switch]{lineno}

\theoremstyle{plain}
\newtheorem{theorem}{Theorem}[section]
\newtheorem{proposition}[theorem]{Proposition}

\newtheorem{corollary}[theorem]{Corollary}
\theoremstyle{definition}
\newtheorem{definition}[theorem]{Definition}

\theoremstyle{remark}

\newcommand*{\tr}{\mathrm{tr}}

\begin{document}

%

%

\twocolumn[

\aistatstitle{A Family of Distributions of Random Subsets for Controlling Positive and Negative Dependence}

\aistatsauthor{ Takahiro Kawashima \And Hideitsu Hino }

\aistatsaddress{ ZOZO Research \And Institute of Statistical Mathematics \\ RIKEN AIP } ]

\begin{abstract}
    Positive and negative dependence are fundamental concepts that characterize
    the attractive and repulsive behavior of random subsets.
    Although some probabilistic models
    are known to exhibit positive or negative dependence,
    it is challenging to seamlessly bridge them with a practicable probabilistic model.
    In this study, we introduce a new family of distributions, named the discrete kernel point process (DKPP), which includes determinantal point processes and parts of Boltzmann machines.
    We also develop some computational methods for probabilistic operations and inference with DKPPs,
    such as calculating marginal and conditional probabilities and learning the parameters.
    Our numerical experiments demonstrate the controllability of positive and negative dependence and the effectiveness of the computational methods for DKPPs.
\end{abstract}

\section{INTRODUCTION}
Random subset selection from a ground set is often encountered in problems related to statistics and machine learning.
One common problem is modeling the purchasing behavior of customers;
buying items from a ground set of products can be seen as the occurrence of a random subset.
To go beyond the independent selection of items, we should consider a probabilistic model on
the powerset of the ground set.
Positive and negative dependence are fundamental concepts that characterize probabilistic models of random subsets.
If the model has positive dependence, an attractive force is present and similar items tend to appear in a random subset.
Conversely, if the model has negative dependence, a repulsive force occurs and random subsets are likely to contain diverse items.

\begin{figure}[t]
    \centering
    \begin{minipage}{0.49\columnwidth}
        \centering
        \includegraphics[width = 1.0\columnwidth]{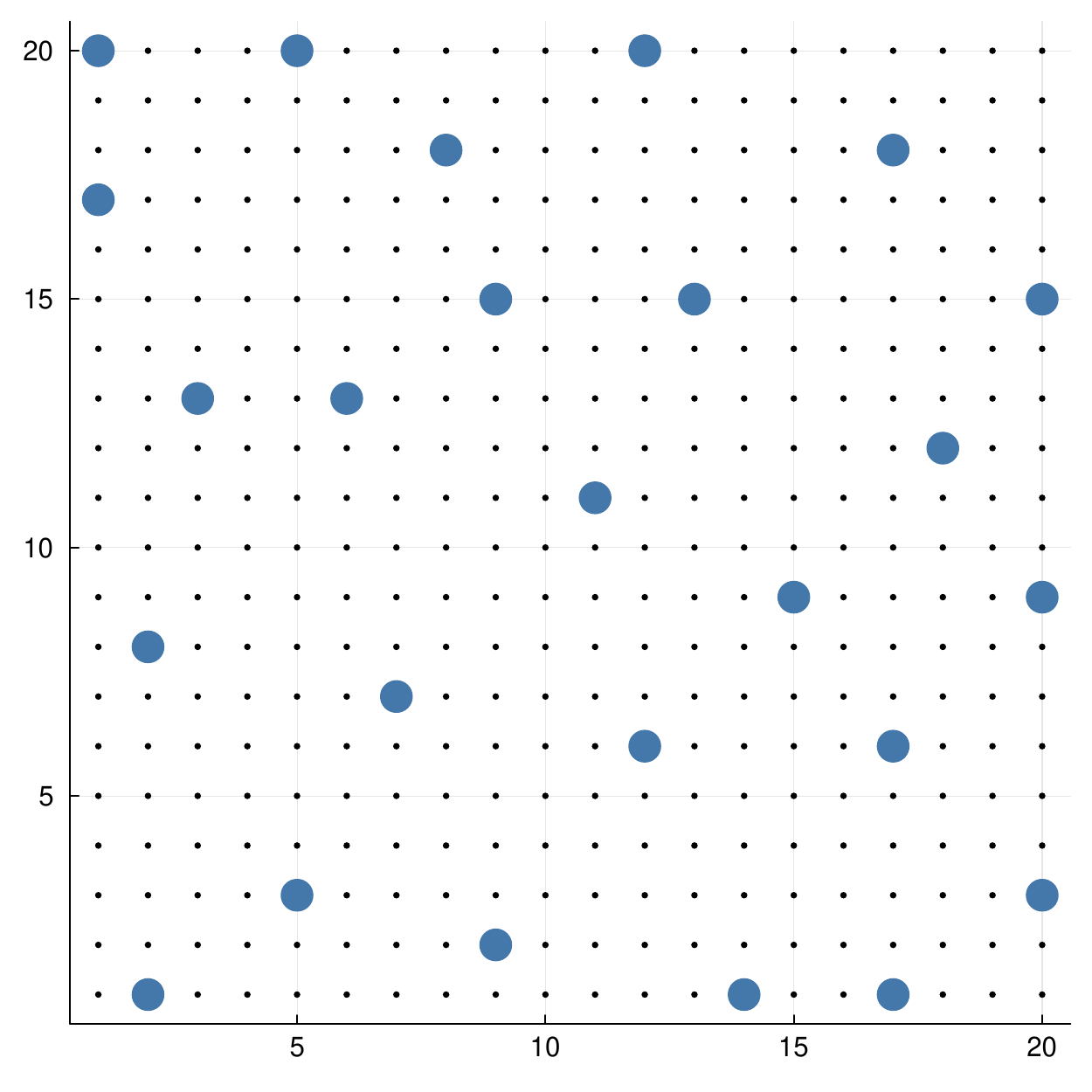}
        \subcaption{
            \texttt{SCATTERED}
        }
        \label{subfig:scattered}
    \end{minipage}
    \begin{minipage}{0.49\columnwidth}
        \centering
        \includegraphics[width = 1.0\columnwidth]{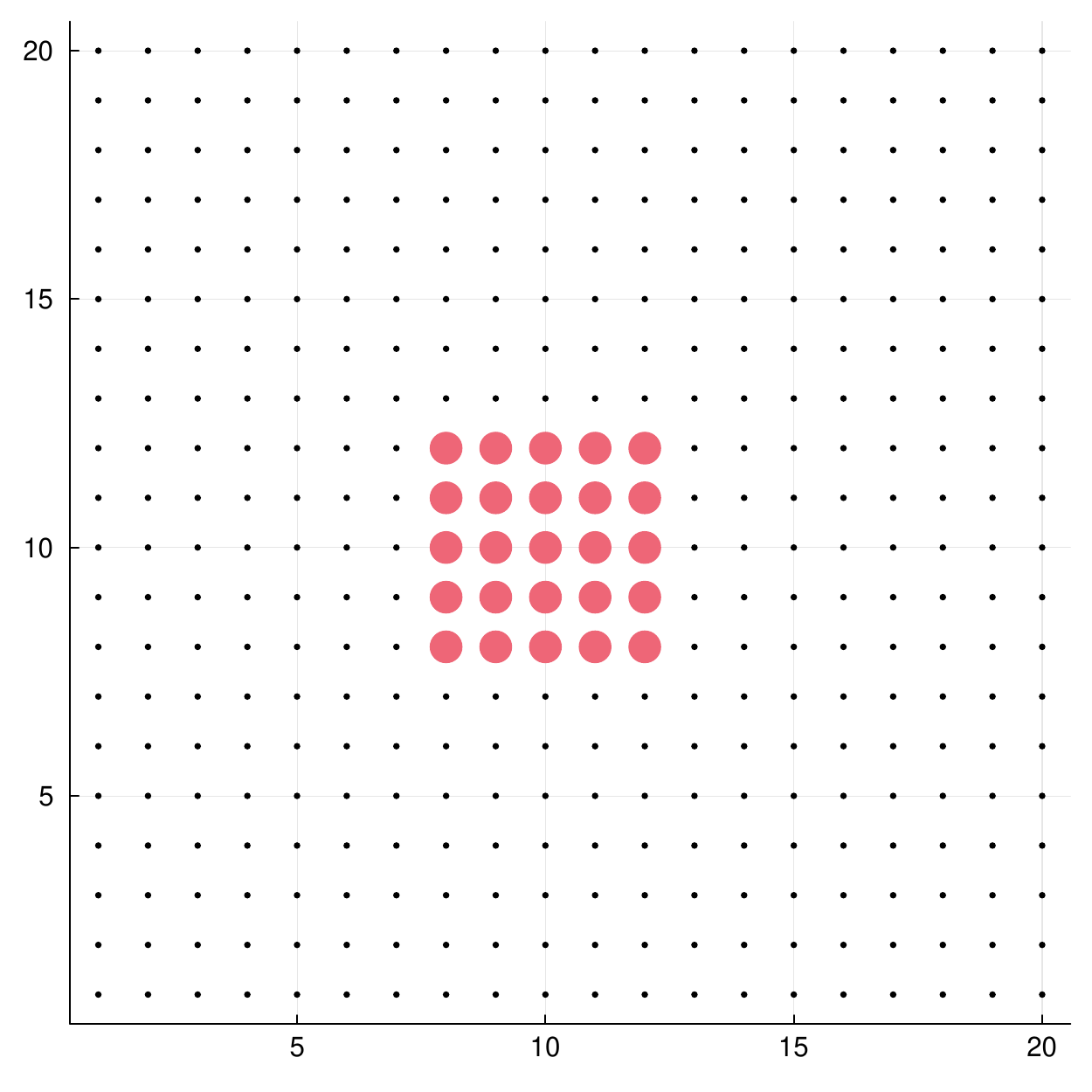}
        \subcaption{
            \texttt{GATHERED}
        }
        \label{fig:gathered}
    \end{minipage}
    \\
    \centering
    \includegraphics[width = 0.8\columnwidth]{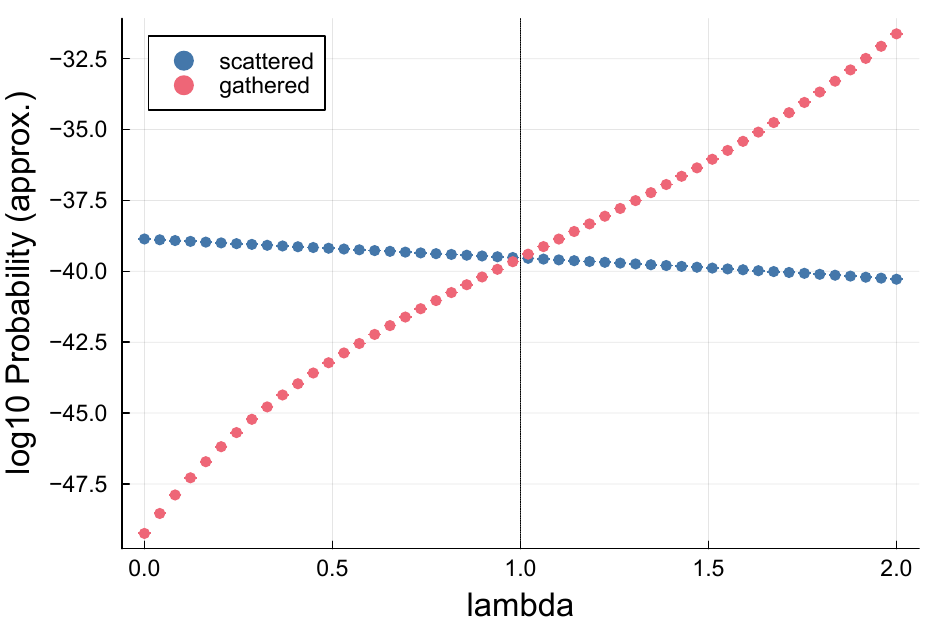}
    \caption{
        Two examples of subsets on the grid and their
        conditional probabilities $\log_{10} P(\mathcal{A} ~|~ \lvert \mathcal{A} \rvert = k)$
        with slight change in function that determines DKPPs.
        See Section \ref{sec:experiments} for more details.
    }
    \label{fig:cond_probs}
\end{figure}

Log-supermodularity and log-submodularity are representative characterizations of positive and negative dependence \citep{fortuin1971, pemantle2000, borcea2009}.
For example, ferromagnetic Ising models and determinantal point processes (DPPs) \citep{macchi1975, borodin2005}
are log-supermodular and log-submodular probabilistic models, respectively.
Since the review article by \citet{kulesza2012},
DPPs have gained increasing attention in the machine learning community owing to their diversity-promoting property.

One of the reasons that makes DPPs popular is their parameterization by a kernel matrix,
which consists of pairwise similarities of items.
Since kernel matrices can be constructed from any features of items,
this parameterization is suitable for various machine learning and statistical problems.
Indeed, DPPs are now applied to widespread problems including
image search \citep{kulesza2011}, document summarization \citep{gillenwater2012a, dupuy2018}, recommender systems \citep{wilhelm2018},
randomized numerical linear algebra \citep{derezinski2020a, derezinski2021},
experimental design \citep{derezinski2020, derezinski2022},
and counterfactual explanation \citep{mothilal2020}.
On the other hand, diversity is not a general prescription for such problems.
Recommending similar glasses to a customer who previously bought certain glasses may not be effective,
but it may be different with socks.
Otherwise, we may want to adjust the strength of the repulsive force.

In this paper, we propose a new family of distributions for random subsets,
discrete kernel point processes (DKPPs), by generalizing DPPs.
DKPPs are determined by a kernel matrix and a scalar function on $\mathbb{R}_{\geq 0}$.
Similarly to DPPs, the kernel matrix provides the pairwise similarity of items.
Furthermore, the presence of the scalar function enables us to control positive and negative dependence.
Figure \ref{fig:cond_probs} shows the actual behavior of DKPPs (see Section \ref{sec:experiments} for details).
Attractive and repulsive forces are flexibly determined by changing a parameter $\lambda$.
We also develop computational methods for evaluating marginal and conditional probabilities
and learning DKPPs for practical use.
Furthermore, we conduct an experiment on repulsive and attractive subset acquisition to demonstrate the applicability of the DKPPs.


\subsection{Related Works}
There are a few previous works that aim to develop a family of distributions
to control positive and negative dependence.
The closest ones to our study are immanantal point processes \citep{diaconis2000} and $\alpha$-DPPs \citep{vere-jones1997},
in which the determinants in DPPs are generalized to the immanant and $\alpha$-determinant, respectively.
Although both immanantal point processes and $\alpha$-DPPs include
permanental point processes (i.e., behave like bosons with positive dependence \citep{macchi1975}),
the computational issue is not resolved;
even computing the permanent is {\#}P-hard \citep{valiant1979}.
For negative dependence only, \citet{mariet2018} considered the exponentiated strong Rayleigh distributions that can control the negative dependence and developed an approximate sampler for them.

\citet{iyer2015a} introduced the submodular and log-submodular point processes as
the family of distributions on a powerset.
Although they mainly discussed how to handle these distributions,
such as probabilistic operations and parameter learning,
they did not focus on the control of positive and negative dependence.


\section{PRELIMINARY}
\subsection{Supermodular and Submodular Functions}
Let $\mathcal{Y} = \{1, \ldots, N\}$ be a finite ground set with $N$ items.
A set function $f: 2^\mathcal{Y} \to \mathbb{R}$ is said to be submodular if
$
    f(\mathcal{S}) + f(\mathcal{T}) \geq f(\mathcal{S}\cup\mathcal{T}) + f(\mathcal{S}\cap\mathcal{T})
$
holds for every $\mathcal{S}, \mathcal{T} \subseteq \mathcal{Y}$.
A set function $f$ is said to be supermodular if $-f$ is submodular
and modular if $f$ satisfies both submodularity and supermodularity.

For a continuous model with a probability density function $p$, the log-concavity of $p$ is considered rather than
concavity \citep{an1996, borzadaran2011}.
In a similar spirit, we will focus on log-supermodular and log-submodular probability functions later;
the set function $f$ has log-supermodularity if $\log f$ is supermodular,
and the same applies to log-submodularity.
If $f$ is log-submodular and log-supermodular, $f$ is said to be a log-modular function.

The multilinear extension of the set function $f$ is one approach
to extend the discrete $f$ to a continuous function and was first introduced for submodular maximization problems \citep{calinescu2007, chekuri2014}.
\begin{definition}
    Let $f:2^\mathcal{Y} \to \mathbb{R}$ be a set function.
    The multilinear extension of $f$, denoted by $\tilde{f}: [0, 1]^N \to \mathbb{R}$, is defined by
    \begin{align}
        \label{eq:multilinear}
        \tilde{f}(\bm{q}) \coloneqq \sum_{\mathcal{A} \subseteq \mathcal{Y}} f(\mathcal{A})
        \prod_{i \in \mathcal{A}} q_i
        \prod_{i \notin \mathcal{A}} (1 - q_i),
    \end{align}
    where $\bm{q} \in [0, 1]^N$.
\end{definition}
Suppose that $Q_{\bm{q}}$ is the probability function of $N$ independent Bernoulli trials with the parameters $\{q_i\}$:
\begin{align}
    \label{eq:variational_distribution}
    Q_{\bm{q}}(\bm{\xi})
    \coloneqq \prod^N_{i=1} Q_{q_i}(\xi_i)
    \coloneqq \prod^N_{i=1} \mathrm{Bernoulli}(\xi_i; q_i),
\end{align}
for $\bm{\xi} \in \{0,1\}^{N}$.
Then, the multilinear extension \eqref{eq:multilinear} can also be written as
\begin{align}
    \label{eq:multilinear_expectation}
    \tilde{f}(\bm{q}) = \mathbb{E}_{\bm{\xi} \sim Q_{\bm{q}}}[f(\mathcal{A}_{\bm{\xi}})],
\end{align}
where $\mathcal{A}_{\bm{\xi}} \coloneqq \{i \in \mathcal{Y} : \xi_i = 1\}$.

\subsection{Positive and Negative Dependence}
Throughout this paper, we consider probability distributions on $2^\mathcal{Y}$,
which assign occurrence probabilities to every subset $\mathcal{A} \subseteq \mathcal{Y}$.
Such distributions are equivalent to those of $N$-dimensional random binary vectors and discrete point processes on $\mathcal{Y}$.

Positive and negative dependence are essential concepts that characterize distributions on $2^\mathcal{Y}$.
Generally, the probability function $P:2^\mathcal{Y} \to [0, 1]$ exhibits positive dependence
when a random subset $\mathcal{A} \sim P$ tends to contain similar elements.
For example, consider a ferromagnetic Ising model.
In this model, closely located spins tend to align in the same direction,
with similarity determined by their distance on the grid.
This scenario illustrates positive dependence.
Conversely, an antiferromagnetic Ising model leads to a random subset
with diverse elements, as adjacent spins tend to have opposite directions.
This is an example of negative dependence.


Log-supermodularity and log-submodularity of the probability function $P$ are representative characterizations
of positive and negative dependence, respectively.
For an intuitive understanding, we consider two singletons $\mathcal{S} = \{i\}$ and $\mathcal{T} = \{j\}$ such that $i \neq j$.
If $P$ is log-submodular, the inequality $P(\{i, j\}) \leq Z P(\{i\}) P(\{j\}) \propto P(\{i\}) P(\{j\})$ holds, where $Z$ is the normalizing constant of $P$.
The co-occurrence of $i$ and $j$ is upper bounded by a constant multiple of the probability that each appears alone, which implies the negative dependence.
The same applies to the positive dependence.
The log-supermodularity of the probability function $P$ is a sufficiently strong condition.
Indeed, the Fortuin–-Kasteleyn–-Ginibre (FKG) inequality states that log-supermodularity leads to other major characterizations \citep{fortuin1971}.
On the other hand, the log-submodularity of $P$ is a relatively weak condition for the negative dependence
induced by other characterizations but not vice versa \citep{pemantle2000, borcea2009}.
However, log-submodular distributions are often used to model diverse random subsets \citep{iyer2015a, tschiatschek2016, djolonga2018} because of their tractability.
In Subsection \ref{subsec:posneg_dkpp}, we also use log-supermodularity and log-submodularity
to control positive and negative dependence.

\subsection{Operator Monotonicity and Convexity}
Let $\phi: \mathbb{R} \to \mathbb{R}$ be a function.
Given the $N \times N$ Hermitian matrix $\bm{X}$ that
can be diagonalized as $\bm{X} = \bm{U} \mathrm{diag}(\lambda_1, \ldots, \lambda_N) \bm{U}^\ast$,
we also regard $\phi$ as a matrix operator such that $\phi: \bm{X} \mapsto \bm{U} \mathrm{diag}(\phi(\lambda_1), \ldots, \phi(\lambda_N)) \bm{U}^\ast$.
Note that a matrix logarithm and a matrix exponential are special cases of $\phi$.
Now, we can define the monotonicity of $\phi$ as an operator.

\begin{definition}\label{def:operator_monotone}
    The function $\phi$ is said to be operator monotone if $\bm{A} \succeq \bm{B}$ implies $\phi(\bm{A}) \succeq \phi(\bm{B})$
    for all $n \in \mathbb{N}$ and for all the $n \times n$ Hermitian matrices $\bm{A}, \bm{B}$.
    The function $\phi$ is operator antitone if $-\phi$ is operator monotone.
\end{definition}
Here, we use $\succeq$ for positive semidefinite ordering.
Since ordinary monotonicity is the special case of $n = 1$ in Definition \ref{def:operator_monotone},
operator monotonicity is much stronger than ordinary monotonicity.
Indeed, $\phi(x) = x^p$ is operator monotone on $[0, \infty)$ for $p \in [0, 1]$, but not for $p = 2$
\citep[see][Chapter 5]{bhatia1997}.

We can also define operator convexity and concavity.
\begin{definition}\label{def:operator_convex}
    The function $\phi$ is said to be operator convex if
    \begin{align}
        t\phi(\bm{A}) + (1 - t) \phi(\bm{B}) \succeq \phi(t\bm{A} + (1-t)\bm{B}),~ t \in [0, 1]
    \end{align}
    holds for all $n \in \mathbb{N}$ and for all the Hermitian matrices $\bm{A}, \bm{B}$.
    A function $\phi$ is operator concave if $-\phi$ is operator convex.
\end{definition}
Let $\phi'$ be a function and $\phi$ be a primitive of $\phi'$.
As with ordinary monotone functions, $\phi$ is operator convex if $\phi'$ is operator monotone.
We denote $\bm{X}[\mathcal{A}] \coloneqq (X_{ij})_{i, j \in \mathcal{A}}$ for $\mathcal{A} \subseteq \mathcal{Y}$.
\citet{friedland2013} obtained an interesting result that bridges operator monotonicity/antitonicity and supermodularity/submodularity.

\begin{theorem}{\citep{friedland2013}}\label{thm:friedland}
    Suppose that $\phi$ is a real continuous function on the interval $\mathcal{E} \subset \mathbb{R}$
    and that $\phi$ is a primitive of the operator monotone function $\phi'$ on $\mathcal{E}$.
    Then, for every $N \times N$ Hermitian matrix $\bm{X}$ whose eigenvalues are all in $\mathcal{E}$, the set function
    \begin{align}
        \label{eq:friedland}
        f: 2^\mathcal{Y} \to \mathbb{R}: \mathcal{A} \mapsto \tr \phi(\bm{X}[\mathcal{A}])
    \end{align}
    is supermodular.
    If $\phi$ is a primitive of the operator antitone $\phi'$, the set function $f$ is submodular.
\end{theorem}
Conceptually, Theorem \ref{thm:friedland} says that the operator convexity/concavity of $\phi$
corresponds to the supermodularity/submodularity of $f: \mathcal{A} \to \bm{X}[\mathcal{A}]$\footnote{This is not always true because there are operator convex functions that cannot be obtained as primitives of operator motonone functions, as stated by \citet{friedland2013}.}.
Note that the submodularity of the set function $\mathcal{A} \mapsto \log\det(\bm{X}[\mathcal{A}])$ is well known \citep{bach2013},
and it is a special case of \eqref{eq:friedland} with $\phi = \log$.
Another example from Theorem \ref{thm:friedland} is $\phi: x \mapsto x^p$;
it leads to a submodular $f$ when $p \in [0, 1]$ and a supermodular $f$ when $p \in [1, 2]$.

\section{DISCRETE KERNEL POINT PROCESSES}
\begin{definition}\label{def:dkpp}
    Let $\phi: \mathbb{R}_{\geq 0} \to \mathbb{R}$ be a continuous function
    and $\bm{L}$ be an $N \times N$ positive semidefinite Hermitian matrix.
    The probability function $P_\phi(\cdot; \bm{L}): 2^\mathcal{Y} \to [0, 1]$ is called a discrete kernel point process (DKPP)
    if it is given by
    \begin{align}
        \label{eq:dkpp}
        &P_\phi(\mathcal{A}; \bm{L}) = \frac{1}{Z_\phi(\bm{L})} \exp \left ( \tr \phi(\bm{L}[\mathcal{A}]) \right )
        \eqqcolon \frac{\tilde{P}_\phi(\mathcal{A}; \bm{L})}{Z_\phi(\bm{L})},\\
        &\quad \mbox{where}\quad Z_\phi(\bm{L}) \coloneqq \sum_{\mathcal{A} \subseteq \mathcal{Y}} \exp \left ( \tr \phi(\bm{L}[\mathcal{A}]) \right )
    \end{align}
    for every $\mathcal{A} \subseteq \mathcal{Y}$.
\end{definition}

Let $\mathcal{C}(\mathbb{R}_{\geq 0}, \mathbb{R})$ denote the set of all functions from $\mathbb{R}_{\geq 0}$ to $\mathbb{R}$.
The set
\begin{align}
    \mathcal{F}_{\mathrm{DKPP}} \coloneqq
    \{P_\phi(\cdot; \bm{L}):~ \phi \in \mathcal{C}(\mathbb{R}_{\geq 0}, \mathbb{R}), \bm{L} \in \mathbb{H}^N_{\geq 0}\}
\end{align}
is a family of specific distributions on $2^\mathcal{Y}$.
We readily see that DPPs belong to $\mathcal{F}_{\mathrm{DKPP}}$.
\begin{proposition}\label{prop:dkpp_dpp}
    If $\phi = \log$, a DKPP $P_\phi(\cdot; \bm{L})$ is a DPP.
\end{proposition}
\begin{proof}
    In general, we have the identity $\tr \log \bm{X} = \log \det \bm{X}$ for every positive semidefinite $\bm{X}$.
    This leads to
    $
        P_{\phi}(\mathcal{A}; \bm{L}) \propto \exp(\tr\log\bm{L}[\mathcal{A}])
        = \exp(\log \det \bm{L}[\mathcal{A}]) = \det\bm{L}[\mathcal{A}].
    $
\end{proof}

One of the representative probabilistic models for a random binary vector $\bm{\xi} \in \{0, 1\}^N$
(equivalently random subsets) is the Boltzmann machine \citep{ackley1985}.
A fully visible Boltzmann machine with bias vectors $\bm{h} \in \mathbb{R}^N$ and symmetric connections $\bm{W}~ (\mathrm{diag}(\bm{W}) = \bm{0},~ W_{ij} = W_{ji})$ is modeled as
\begin{align}
    \label{eq:boltzmann_machine}
    P_{\mathrm{BM}}(\bm{\xi}; \bm{h}, \bm{W})
    \propto \exp \left ( \sum^N_{i = 1} h_i \xi_i + \sum^N_{i, j = 1} W_{ij} \xi_i \xi_j \right ).~~
\end{align}
In the Boltzmann machine \eqref{eq:boltzmann_machine}, two-body interactions of $\bm{\xi}$ are captured by the quadratic term.
We can see that a DKPP becomes a Boltzmann machine if $\phi$ is a quadratic function.
\begin{proposition}\label{prop:dkpp_bm}
    Suppose that $\phi(x) = ax^2 + bx + c$. Then, the DKPP $P_\phi(\cdot; \bm{L})$ is equivalent to
    the Boltzmann machine \eqref{eq:boltzmann_machine} with the parameters
    \begin{align}
        \label{eq:l_to_w}
        W_{ij} &= \left \{~
        \begin{aligned}
        & 0 & \quad (i = j), \\
        & a \lvert L_{ij} \rvert^2 & \quad (i \neq j),
        \end{aligned}
        \right .\\
        \label{eq:l_to_h}
        h_i &= aL^2_{ii} + bL_{ii} + c,
    \end{align}
    for $i, j = 1, \ldots, N$.
\end{proposition}
The proof of Proposition \ref{prop:dkpp_bm} is in Appendix \ref{app:proofs}.
Note that although any values are allowed for $W_{ij}$ in Boltzmann machines,
the relation \eqref{eq:l_to_w} implies that a DKPP with quadratic $\phi$ can only represent 
either all non-negative or all non-positive $\{W_{ij}\}_{i, j}$.

When $a = 0$, the connections $W_{ij}$ vanish in \eqref{eq:l_to_w},
meaning all elements of $\bm{\xi}$ become independent.
Formally, the following corollary holds as a special case of Proposition \ref{prop:dkpp_dpp}.
\begin{corollary}\label{prop:dkpp_bernoulli}
    Suppose that $\phi(x) = bx + c$. Then, the DKPP $P_\phi(\cdot; \bm{L})$ is equivalent to
    $N$ independent Bernoulli trials where the probability of success of the $i$-th trial is given by
    $p_i = \sigma(b L_{ii} + c)$, with $\sigma: \mathbb{R} \to (0, 1)$ being the logistic sigmoid function.
\end{corollary}

Additionally, the following proposition also holds for affine $\phi$.
\begin{proposition}\label{prop:affine_logmod}
    A DKPP $P_\phi(\cdot; \bm{L})$ is log-modular if and only if $\phi$ is affine for all $x \in \mathbb{R}_{\geq 0}$.
\end{proposition}
The proof of Proposition \ref{prop:affine_logmod} is shown in Appendix \ref{app:proofs}.

\subsection{Positive and Negative Dependence of DKPPs}\label{subsec:posneg_dkpp}
The behavior of a DKPP $P_\phi$ is fundamentally determined by $\phi$,
allowing control over positive and negative dependence by appropriately choosing $\phi$.
The following corollary follows directly from Theorem \ref{thm:friedland}.

\begin{corollary}\label{cor:dkpp_posneg}
    Let $\phi$ be a primitive of a function $\phi'$.
    A DKPP $P_\phi(\cdot; \bm{L})$ is log-supermodular if $\phi'$ is operator monotone
    and log-submodular if $\phi'$ is operator antitone.
\end{corollary}
An appropriate parameterization of $\phi$ enables a smooth transition between positive and negative dependence.
For example, we consider the scaled Box--Cox transformation for $\phi$:
\begin{align}
    \label{eq:box_cox}
    \phi_{\beta, \lambda}(x) \coloneqq
    \left \{ 
    \begin{aligned}
        &\beta \log x &~ (\lambda = 0),\\
        &\frac{\beta(x^\lambda - 1)}{\lambda} &~ (\mathrm{otherwise}),
    \end{aligned}
    \right .
\end{align}
with the hyperparameters $\lambda \in \mathbb{R}$ and $\beta \in \mathbb{R}_{>0}$.
According to Corollary \ref{cor:dkpp_posneg}, a DKPP $P_{\phi_{\beta, \lambda}}$ exhibits negative dependence for $\lambda \in [0, 1]$
and positive dependence for $\lambda \in [1, 2]$.
It reduces to a DPP for $\lambda = 0$ and a Boltzmann machine for $\lambda = 2$.
Hereafter, we denote $\phi_{\lambda} \coloneqq \phi_{1, \lambda}$ for simplicity.

\section{OPERATIONS AND INFERENCE OF DKPPs}
\subsection{Mode Exploration}
The mode exploration
$\mathrm{arg}\max_{\mathcal{A} \subseteq \mathcal{Y}} \log P_\phi(\mathcal{A}; \bm{L})$
is one of the most fundamental problems regarding probability model over sets,
but it is NP-hard for a general $\phi$ for DKPPs.
If $P_\phi(\cdot; \bm{L})$ is log-supermodular, it becomes a submodular minimization problem.
Although there are (strongly) polynomial-time combinatorial algorithms \citep{schrijver2000, iwata2001, orlin2009} for submodular minimization problems,
their computational complexity is expensive (typically $\mathcal{O}(N^6)$ for function calls).
The minimum-norm point algorithm \citep{fujishige2006, fujishige2011, chakrabarty2014} is an alternative to such combinatorial algorithms.
Although the minimum-point algorithm has weaker theoretical complexity, it usually performs better in practice.

When $P_\phi(\cdot; \bm{L})$ is log-submodular, the problem becomes a submodular maximization problem.
Although submodular maximization problems are NP-hard, many approximation algorithms have been proposed.
If $\log P_\phi(\cdot; \bm{L})$ is not only submodular but also monotone
(i.e., $\log P_\phi(\mathcal{S}; \bm{L}) \leq \log P_\phi(\mathcal{T}; \bm{L})$ holds for every $\mathcal{S} \subseteq \mathcal{T}$), the simple greedy algorithm
provides a $1-1/e$ approximation \citep{nemhauser1978}.
However, we have no guarantee of monotonicity in general.
For example, a DPP is a special case of DKPPs and has a submodular but non-monotone log-probability function.
For general non-monotone submodular maximization problems,
a deterministic 1/3-approximate algorithm and a randomized 1/2-approximate algorithm (in expectation) are proposed in \citep{buchbinder2012}.

In constrained settings, submodular minimization problems are generally NP-hard \citep{garey1979, feige2001},
and there are no algorithms with a polynomial approximation factor
even for constraints with a cardinality lower bound \citep{svitkina2011}.
Nevertheless, some practical approximation techniques can be applied \citep{svitkina2011, iyer2013}.
Conversely, there are approximate algorithms yielding constant factor approximations
for submodular maximization problems under cardinality-constrained settings.
\citet{buchbinder2014} proposed efficient algorithms achieving the approximation factors
in the range $[1/e + 0.004, 1/2 - o(1)]$ for the constraint $\lvert \mathcal{A} \rvert \leq k$
and $[0.356, 1/2 - o(1)]$ for $\lvert \mathcal{A} \rvert = k$, in expectation.

\subsection{Sampling}
Although direct sampling from a DKPP $P_\phi(\cdot; \bm{L})$ is difficult for a general $\phi$,
Markov chain Monte Carlo (MCMC) samplers are effective.
Although many advanced MCMC samplers are not suitable for DKPPs owing to their discreteness,
the recently proposed Langevin-like sampler \citep{zhang2022} may be a viable option.
Classical MCMC samplers are also applicable.
For the probability function $P$ on the powerset $2^V$, mixing times of the Metropolis--Hastings sampler
(including the Gibb sampler) are studied both in the log-supermodular or log-submodular cases \citep{gotovos2015} and for general cases \citep{rebeschini2015}.

\subsection{Normalizing Constant and Expectation}\label{subsec:normalize}
The evaluation of the normalizing constant $Z_\phi (\bm{L})$ is also crucial,
as it is required for calculating marginal probabilities.
Several methods are developed for approximating or bounding $Z_\phi(\bm{L})$,
including the mean-field approximation for set distributions \citep{djolonga2018} and
the perturb-and-MAP method \citep{papandreou2011, hazan2012, balog2017}.
The L-field developed in \citep{djolonga2014, djolonga2018} is also applicable to obtain
the lower and upper bound of $\log Z_\phi(\bm{L})$ if the DKPP is log-submodular or log-supermodular.

The mean-field approximation is a basic method in statistical mechanics and Bayesian statistics to approximate target distributions.
For the DKPP $P_\phi(\cdot; \bm{L})$, we can employ $Q_{\bm{q}}$ defined in \eqref{eq:variational_distribution} for the variational distribution
and aim to minimize the Kullback--Leibler divergence (KLD) between the variational distribution and the DKPP.
Then, we obtain the following iterative update rule for $i = 1, \ldots, N$
from the coordinate ascent:
\begin{align}
    \label{eq:mean_field_update}
    &q_i \gets \sigma \left ( \mathbb{E}_{\bm{\xi}_{\backslash i} \sim Q_{\bm{q}_{\backslash i}}} [f(i \vert \mathcal{A}_{\bm{\xi}_{\backslash i}})] \right ),\\
    &\quad \mbox{where}\quad 
    Q_{\bm{q}_{\backslash i}}(\bm{\xi}_{\backslash i}) \coloneqq \prod_{j \neq i} \mathrm{Bernoulli}(\xi_j; q_j),\\
    &\hphantom{\quad \mbox{where}\quad}
    f(i \vert \mathcal{A}) \coloneqq \tr \phi (\bm{L}[\mathcal{A} \cup \{i\}]) - \tr \phi(\bm{L}[\mathcal{A}]).
\end{align}
Here, $\mathcal{A}_{\bm{\xi}_{\backslash i}}$ is defined as 
$\mathcal{A}_{\bm{\xi}_{\backslash i}} \coloneqq \{i \in\{1, \ldots, i-1, i+1, \ldots, N\}: \xi_i = 1\}$.
The derivation of \eqref{eq:mean_field_update} is presented in Appendix \ref{app:mean_field}.
The expectation in \eqref{eq:mean_field_update} can easily be approximated using Monte Carlo methods.
Once the optimal $\bm{q}$ is found, the tightened evidence lower bound (ELBO)
$L(\bm{q}) \coloneqq \mathbb{H}[Q_{\bm{q}}] + \mathbb{E}_{\bm{\xi} \sim Q_{\bm{q}}}[\tr\phi (\bm{L}[\mathcal{A}_{\bm{\xi}}])]$,
which ensures $\log Z_{\phi} \geq L(\bm{q})$,
can be computed ($\mathbb{H}[\cdot]$ means the entropy).

Our proposed method for evaluating the normalizing constant $Z_\phi(\bm{L})$ of a DKPP is
the combination of mean-field approximation and importance sampling.
For any set function $g$ and a proposal distribution $Q$ on $2^\mathcal{Y}$,
the expectation $\mathbb{E}_{\mathcal{A} \sim P_\phi}[g(\mathcal{A})]$ can be evaluated as
the weighted mean over $Q$ since
\begin{align}
    \label{eq:importance_sampling}
    \mathbb{E}_{\mathcal{A} \sim P_\phi}[g(\mathcal{A})]
    = \mathbb{E}_{\mathcal{A} \sim Q}[w(\mathcal{A}) g(\mathcal{A})],
\end{align}
where $w(\mathcal{A}) \coloneqq P_\phi(\mathcal{A}; \bm{L}) / Q(\mathcal{A})$.
For unnormalized $\tilde{P}_\phi$, 
\begin{align}
    &1 = \mathbb{E}_{\mathcal{A} \sim P_\phi}[1]
    = \frac{1}{Z_\phi(\bm{L})} \mathbb{E}_{\mathcal{A} \sim Q}\left [ \frac{\tilde{P}_\phi(\mathcal{A}; \bm{L})}{Q(\mathcal{A})}\right ]\\
    \label{eq:importance_sampling_Z}
    &\quad\Longleftrightarrow Z_\phi(\bm{L}) = \mathbb{E}_{\mathcal{A} \sim Q}\left [ \frac{\tilde{P}_\phi(\mathcal{A}; \bm{L})}{Q(\mathcal{A})}\right ]
\end{align}
yields the sampling-based method for evaluating the normalizing constant.
If our goal is to obtain \eqref{eq:importance_sampling},
the approximated $Z_\phi(\bm{L})$ can be plugged into the weight
$w(\mathcal{A}) = P_\phi(\mathcal{A}; \bm{L}) / Q(\mathcal{A}) = \tilde{P}_\phi(\mathcal{A}; \bm{L}) / (Z_\phi(\bm{L}) Q(\mathcal{A}))$.

When evaluating \eqref{eq:importance_sampling_Z}, i.e., $g(\mathcal{A}) \equiv 1$,
we employ the variational distribution $Q_{\bm{q}}$ obtained by iterating \eqref{eq:mean_field_update}
as the proposal distribution.
In general, $Q^\ast(\mathcal{A}) \propto \lvert g(\mathcal{A}) \rvert P_\phi(\mathcal{A}; \bm{L})$ minimizes $\mathrm{Var}_{\mathcal{A} \sim Q} [w(\mathcal{A}) g(\mathcal{A})]$ \citep{rubinstein2008}.
Because this choice of the proposal distribution is the solution of
$
\mathrm{arg}\min_{Q_{\bm{q}}} \mathrm{KL}(Q_{\bm{q}} \Vert Q^\ast)
= \mathrm{arg}\min_{Q_{\bm{q}}} \mathrm{KL}(Q_{\bm{q}} \Vert P_\phi(\cdot; \bm{L}))
$,
it is the reasonable choice for evaluating \eqref{eq:importance_sampling_Z}.

\subsection{Marginal and Conditional Probabilities}
Consider recommending $k$ items as a subset of size $k$.
The customer may already have items in the basket; then conditional probabilities naturally arise.
Marginal probabilities are also important because they define conditional probabilities.
We introduce some computational techniques to evaluate marginal and conditional probabilities of random subset models.

\subsubsection*{Marginal Probability}
Let $\mathcal{A}_{\mathrm{sub}}$ and $\mathcal{A}_{\mathrm{sup}}$ be given subsets $\mathcal{Y}$
such that $\mathcal{A}_{\mathrm{sub}} \subseteq \mathcal{A}_{\mathrm{sup}}$.
We consider approximating the marginal probability
$\mathbb{P}(\mathcal{A}_{\mathrm{sub}} \subseteq \mathcal{A} \subseteq \mathcal{A}_{\mathrm{sup}})$
when $\mathcal{A}$ is a random subset following a probabilistic function on $2^\mathcal{Y}$.
A straightforward way to evaluate this marginal probability is
to use the importance sampling as in \eqref{eq:importance_sampling}:
\begin{align}
    \mathbb{P}(\mathcal{A}_{\mathrm{sub}} \subseteq \mathcal{A} \subseteq \mathcal{A}_{\mathrm{sup}})
    &= \mathbb{E}_{P_\phi}[\mathds{1} (\mathcal{A}_{\mathrm{sub}} \subseteq \mathcal{A} \subseteq \mathcal{A}_{\mathrm{sup}})]\\
    \label{eq:marginal_is}
    &= \mathbb{E}_{Q}[w(\mathcal{A}) \mathds{1} (\mathcal{A}_{\mathrm{sub}} \subseteq \mathcal{A} \subseteq \mathcal{A}_{\mathrm{sup}})].
\end{align}
The Monte Carlo method for \eqref{eq:marginal_is} is unbiased for 
$\mathbb{P}(\mathcal{A}_{\mathrm{sub}} \subseteq \mathcal{A} \subseteq \mathcal{A}_{\mathrm{sup}})$,
but the critical issue remains:
most Monte Carlo samples yield zero when
$\lvert \mathcal{A}_{\mathrm{sup}}\backslash\mathcal{A}_{\mathrm{sub}} \rvert$
is small owing to the indicator function.
If we use Monte Carlo samples on
$2^{\mathcal{A}_{\mathrm{sup}}\backslash\mathcal{A}_{\mathrm{sub}}}$
instead of $2^{\mathcal{Y}}$, the number of wasted samples could be reduced.
The following proposition realizes this.
\begin{proposition}\label{prop:marginal_rao_blackwell}
    Let $P$ be a probability function on $2^\mathcal{Y}$
    and $\bm{\xi}$ be a random vector following the independent Bernoulli trials $Q_{\bm{q}}$ defined in \eqref{eq:variational_distribution}.
    Then, 
    \begin{align}
        &\mathbb{E}_{\bm{\xi}}\left [
        \frac{P(\mathcal{A}_{\bm{\xi}})}{\prod_{i \in \mathcal{A}_{\mathrm{sup}}\backslash\mathcal{A}_{\mathrm{sub}}} Q_{q_i}(\xi_i)}
        \left \vert
        \begin{aligned}
            & \xi_i = 1~ &(i \in \mathcal{A}_{\mathrm{sub}})\\
            & \xi_j = 0~ &(j \in \mathcal{Y} \backslash \mathcal{A}_{\mathrm{sup}})
        \end{aligned}
        \right . \right]\\
        \label{eq:rao_blackwellized}
        &\quad= 
        \mathbb{P}(\mathcal{A}_{\mathrm{sub}} \subseteq \mathcal{A} \subseteq \mathcal{A}_{\mathrm{sup}})
    \end{align}
    and
    \begin{align}
        &\mathrm{Var}_{\bm{\xi}}\left [
        \frac{P(\mathcal{A}_{\bm{\xi}})}{\prod_{i \in \mathcal{A}_{\mathrm{sup}}\backslash\mathcal{A}_{\mathrm{sub}}} Q_{q_i}(\xi_i)}
        \left \vert
        \begin{aligned}
            & \xi_i = 1~ &(i \in \mathcal{A}_{\mathrm{sub}})\\
            & \xi_j = 0~ &(j \in \mathcal{Y} \backslash \mathcal{A}_{\mathrm{sup}})
        \end{aligned}
        \right . \right]\\
        &\quad\leq 
        \mathrm{Var}_{\bm{\xi}}[w(\mathcal{A}_{\bm{\xi}}) \mathds{1} (\mathcal{A}_{\mathrm{sub}} \subseteq \mathcal{A}_{\bm{\xi}} \subseteq \mathcal{A}_{\mathrm{sup}})]
    \end{align}
    hold, where $w(\mathcal{A}_{\bm{\xi}}) = P(\mathcal{A}_{\bm{\xi}})/Q_{\bm{q}}(\bm{\xi})$.
\end{proposition}
The proof of Proposition \ref{prop:marginal_rao_blackwell} mainly depends on
the tower properties of expectation and variance
(the detailed proof is in Appendix \ref{app:proofs}).
Such variance reduction techniques are called Rao--Blackwellization and are commonly used in computational statistics \citep{casella1996, doucet2000}.
Since only $\{\xi_i\}_{i \in \mathcal{A}_{\mathrm{sup}}\backslash\mathcal{A}_{\mathrm{sub}}}$ 
are required as Monte Carlo samples owing to the independence of the proposal distribution $Q_{\bm{q}}$,
and the indicator function no longer appears in \eqref{eq:rao_blackwellized},
we can expect a significant reduction in variance when evaluating the marginal probability.
Rao--Blackwellization is effective in estimating
marginal probabilities other than
$\mathbb{P}(\mathcal{A}_{\mathrm{sub}} \subseteq \mathcal{A} \subseteq \mathcal{A}_{\mathrm{sup}})$.

\begin{proposition}\label{prop:marginal_cardinality}
    Let $\mathcal{A}$ be a random subset following
    the probability function
    $P: 2^\mathcal{Y} \to [0, 1]$.
    Then, we have
    \begin{align}
        \label{eq:rao_blackwell_cardinality}
        \mathbb{P}(\lvert\mathcal{A}\rvert = k)
        = \binom{N}{k} \mathbb{E}_{\mathcal{A}^k \sim Q^k}[P(\mathcal{A}^k)],
    \end{align}
    where $Q^k$ is the uniform distribution on
    $\{\mathcal{A} \subseteq \mathcal{Y} : \lvert \mathcal{A} \rvert = k\}$.
\end{proposition}
The proof of \ref{prop:marginal_cardinality} is also shown in
Appendix \ref{app:proofs}.
Proposition \ref{prop:marginal_cardinality} provides
an accurate Monte Carlo method to evaluate the marginal probability
$\mathbb{P}(\lvert \mathcal{A} \rvert = k)$ for a given $k \in \{1, \ldots, N\}$.
Simply take Monte Carlo samples consisting of uniformly randomly chosen $k$ items in $\mathcal{Y}$
and evaluate \eqref{eq:rao_blackwell_cardinality}.

\subsubsection*{Conditional Probability}
The conditional probability
\begin{align}
    \label{eq:cond_prob}
    P_\phi(\mathcal{A} \vert \mathcal{A}_{\mathrm{sub}} \subseteq \mathcal{A} \subseteq \mathcal{A}_{\mathrm{sup}}; \bm{L})
    = \frac{P_\phi(\mathcal{A}; \bm{L})}{\mathbb{P}(\mathcal{A}_{\mathrm{sub}} \subseteq \mathcal{A} \subseteq \mathcal{A}_{\mathrm{sup}})}\quad
\end{align}
can be accurately estimated by approximating the marginal probability
$\mathbb{P}(\mathcal{A}_{\mathrm{sub}} \subseteq \mathcal{A} \subseteq \mathcal{A}_{\mathrm{sup}})$
using Proposition \ref{prop:marginal_rao_blackwell}.
Note that the approximation of the normalizing constants within \eqref{eq:rao_blackwellized} is
not required if we are only interested in the conditional probability
because they cancel out in the numerator and denominator in \eqref{eq:cond_prob}.
The same applies to
$P_\phi(\mathcal{A} \vert \lvert \mathcal{A} \rvert = k; \bm{L})$.

\section{LEARNING DKPPs}\label{sec:learning}
Since the kernel matrix $\bm{L}$ of DKPPs indicates the pairwise similarity for each $(i, j) \in \mathcal{Y} \times \mathcal{Y}$,
we can build it directly when the items have feature vectors.
Otherwise, we may need to learn the kernel matrix $\bm{L}$ from observations.

Let $\mathcal{A}_1, \ldots, \mathcal{A}_M \in \mathcal{Y}$ be the observed data series with $M$ samples.
The simplest way to learn $\bm{L}$ is through the maximum likelihood estimation (MLE).
For $\phi = \log$, say for DPPs, several algorithms have been developed to learn $\bm{L}$ through MLE \citep{gillenwater2014, mariet2015, mariet2016, gartrell2017, dupuy2018, kawashima2023}.
However, the term $\log Z_\phi(\bm{L})$ makes it difficult to learn $\bm{L}$ directly for general $\phi$.

As an alternative to MLE, we propose using ratio matching, which was developed by \citep{hyvarinen2007} as a discrete extension of score matching \citep{hyvarinen2005}.
According to the prescription of the ratio matching,
the objective function to minimize is
\begin{align}
    \label{eq:ratio_matching_dkpp}
    &J(\bm{L})
    \coloneqq \frac{1}{M} \sum_{m, n}
    g \left( \exp(\tr \phi(\bm{L}[\mathcal{A}_m]) - \tr\phi(\bm{L}[\mathcal{A}^{\bar{n}}_m] ))\right )^2,\\
    &\mbox{where}\quad 
    g(x) \coloneqq \frac{1}{1+x},~
    \mathcal{A}^{\bar{n}} \coloneqq
    \left \{
    \begin{aligned}
        &\mathcal{A} \backslash \{n\} & ~~ (n \in \mathcal{A}),\\
        &\mathcal{A} \cup \{n\} & ~~ (n \notin \mathcal{A}),
    \end{aligned}
    \right .
\end{align}
for DKPPs.
Although we cannot evaluate the exponent term in \eqref{eq:ratio_matching_dkpp}
more efficiently than by direct computation for general $\phi$,
we can employ stochastic gradient descent (SGD) because $J$ has a summation form.
Let $\Omega \subseteq \{(m, n) : m = 1, \ldots, M,~ n = 1, \ldots, N\}$ be the minibatch
and $\kappa \coloneqq \max\{\lvert \mathcal{A}_1 \rvert, \ldots, \lvert \mathcal{A}_M \rvert\}$.
The time complexity of computing the gradient $\partial J / \partial \bm{L}$
is $\mathcal{O}(\lvert \Omega \rvert \kappa^3 + \lvert \Omega \rvert N^2)$, which no longer depends on
the sample size $M$ and the number of items $N$,
except for the term $\lvert \Omega \rvert N^2$ owing to
$\lvert \Omega \rvert$ additions of $N \times N$ matrices (see Appendix \ref{app:gradient} for derivation).
This ensures the good scalability of the algorithm.
Optionally, the variance reduction technique for SGD of the ratio matching can be applied \citep{liu2022}.

\section{EXPERIMENTS}\label{sec:experiments}
All experiments were performed on a MacBook Pro (2023, macOS 14.3) with an Apple M3 Pro chip and 36GB RAM.

\subsection{Controllability of Positive and Negative Dependence}
To verify that DKPPs can flexibly control positive and negative dependence,
we experiment on the behavior of the conditional probability $P_\phi(\mathcal{A} ~\vert~ \lvert \mathcal{A} \rvert = k; \bm{L})$.
We prepare a $20 \times 20$ grid of points in $\mathbb{R}^2$ as the ground set $\mathcal{Y}$,
so that $N = 20^2 = 400$.
The kernel matrix $\bm{L}$ is constructed from the Gaussian kernel with unit bandwidth.
Then, we create two types of realization on $2^\mathcal{Y}$: \texttt{SCATTERED} and \texttt{GATHERED},
both containing $k = \lvert \mathcal{A} \rvert = 25$ items.
\texttt{SCATTERED} is obtained by applying an approximate algorithm
for submodular maximization \citep{buchbinder2014} to a DPP,
whereas \texttt{GATHERED} is created manually.
We evaluate the conditional probabilities
$P_{\phi_\lambda}(\mathcal{A} ~\vert~ \lvert \mathcal{A} \rvert = k; \bm{L})$
where $\phi_\lambda$ is the Box--Cox transformation \eqref{eq:box_cox}.
We execute 30 trials of Monte Carlo approximations using \eqref{eq:rao_blackwell_cardinality}
with \num{1000} samples to estimate the conditional probabilities.

The results are shown in Figure \ref{fig:cond_probs}.
The error bars are present but nearly invisible owing to the remarkably small approximation variance.
As expected, the probability of \texttt{SCATTERED} decreases monotonically
and the probability of \texttt{GATHERED} increases monotonically with $\lambda$.
Interestingly, the probabilities of \texttt{SCATTERED} and \texttt{GATHERED} reverse at $\lambda = 1$.
This result aligns exactly with the theory;
the switch between positive and negative dependence occurs at $\lambda = 1$, as explained in Subsection \ref{subsec:posneg_dkpp}.

\subsection{Subset Acquisition}
\begin{algorithm}[t]
    \SetKwComment{Comment}{//}{}
    \caption{Subset Acquiring Algorithm}\label{alg:acquiring}
    \KwIn{DKPP $P_\phi(\cdot; \bm{L})$ and the subset size $k$}
    \KwOut{$\mathcal{A}$}
    Initialize $\mathcal{A} \gets \lbrace \rbrace,~ \mathcal{B} \gets \lbrace1, \ldots, N\rbrace$;\\
    \For{$t = 1$ to $k$}{
        Sample $j \in \mathcal{B}$ with probability proportional to weights $w_i = \tilde{P}_\phi(\mathcal{A} \cup \{i\} | \boldsymbol{L}),~ i \in \mathcal{B}$;\\
        Update $\mathcal{A} \gets \mathcal{A} \cup \{i\},~ \mathcal{B} \gets \mathcal{B} \backslash \{j\}$;
    }
\end{algorithm}
\begin{figure}[t]
    \centering
    \includegraphics[width = 0.75\columnwidth]{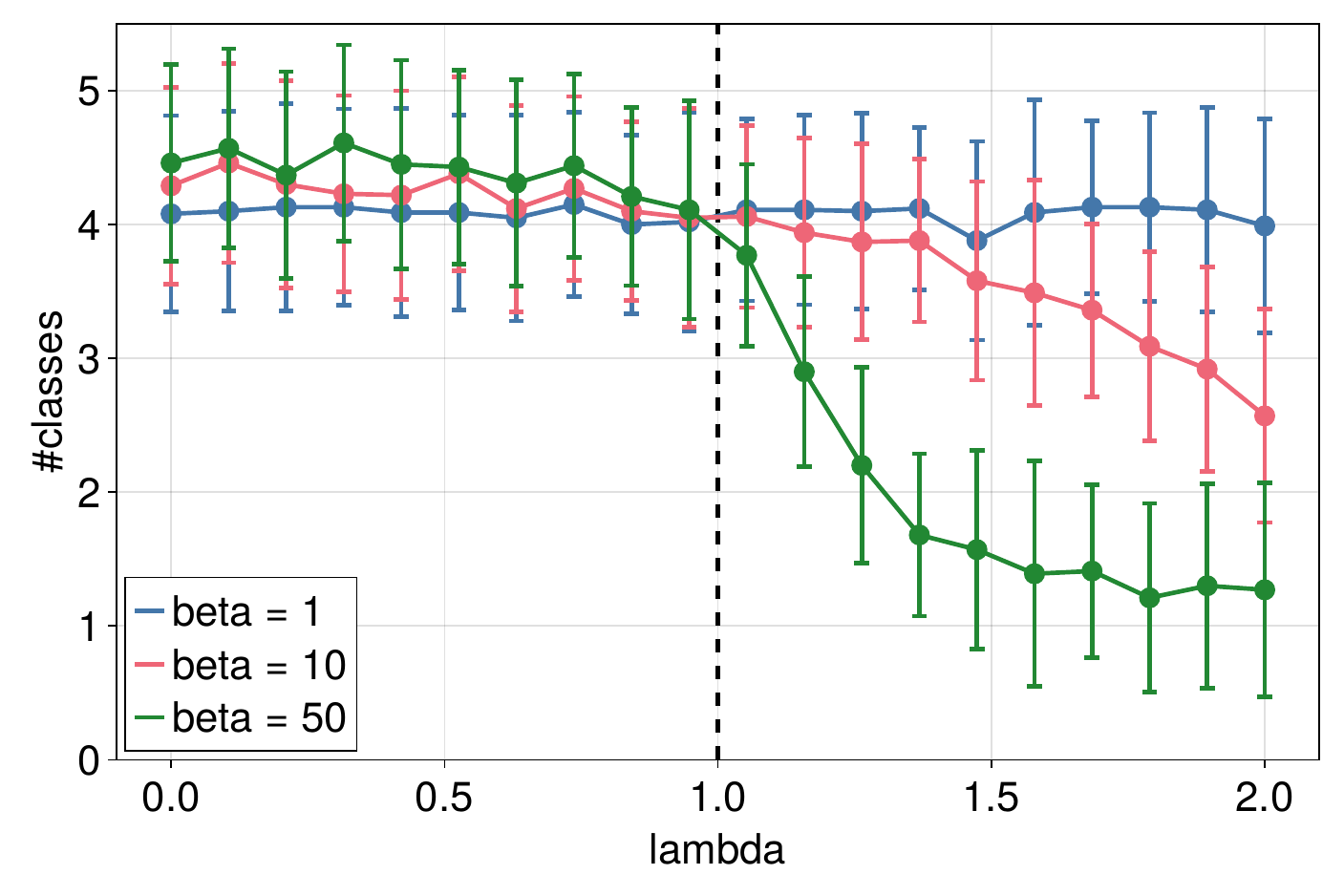}
    \caption{Number of distinct classes within the acquired subsets.}
    \label{fig:mnist_metrics}
\end{figure}

Attractive and repulsive subset acquisition may be the most direct application of the DKPPs.
We randomly pick $N = \num{2000}$ images of the MNIST dataset \citep{lecun1998} with containing $200$ instances from each digit class.
Then, the Gaussian kernel with the median heuristic is applied to make the kernel matrix $\bm{L} \in \mathbb{R}^{\num{2000} \times \num{2000}}$.
Algorithm \ref{alg:acquiring} is the heuristic algorithm to acquire subsets from a DKPP.
Given the kernel matrix, we run the algorithm to acquire subsets of size $5$ from $P_{\phi_{\beta, \lambda}}$
and measure the number of distinct digit classes within each subset.
A larger value of this metric indicates more repulsive subset acquisition, while a smaller value suggests more attractive.

Figure \ref{fig:mnist_metrics} shows the metric values for $\beta \in \{1, 10, 50\},~ \lambda \in [0, 2]$ with the mean and standard deviations of 100 trials.
Although the values are not significantly different in $\beta = 1$, it smoothly decreases as $\lambda$ increases for $\beta = 10$;
the dependency control is realized.
We can find the dependency is more intensified for $\beta = 50$.
Additionally, we visually show the acquired instances in Appendix \ref{app:experiments}.

Notably this procedure involves no learning process at all.
The entire process is forward-only, demonstrating the advantage of the DKPP’s clear and practicable parameterization via the kernel matrix and the dependency controlling function $\phi$.

\subsection{Approximating the Normalizing Constant}
We assess some approximating or bounding methods for $\log Z_\phi(\bm{L})$.
We use the Box--Cox transformation \eqref{eq:box_cox} for $\phi$ and the hyperparameter $\lambda \in [0, 2]$.
We calculate the mean-field ELBOs and demonstrate the importance sampling described in Subsection \ref{subsec:normalize}.
Additionally, we apply the L-field \citep{djolonga2014, djolonga2018} to obtain both the upper and lower bounds of $\log Z_\phi(\bm{L})$
since the DKPP is log-submodular or log-supermodular for such $\lambda$.
For all experiments reported in this subsection, we ran 30 trials and showed the means and standard deviations in the figures.

\begin{figure}[t]
    \centering
    \includegraphics[width = 0.75\columnwidth]{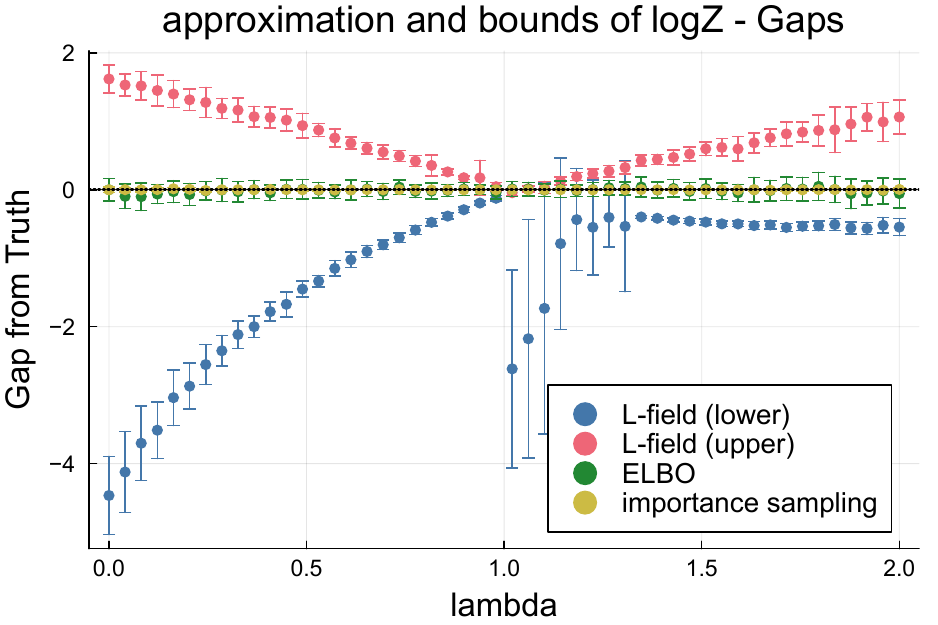}
    \caption{
    Evaluated gaps $\log Z^{\mathrm{approx}}_\phi(\bm{L}) - \log Z_\phi(\bm{L})$.
    }
    \label{fig:Z_gaps}
    \includegraphics[width = 0.75\columnwidth]{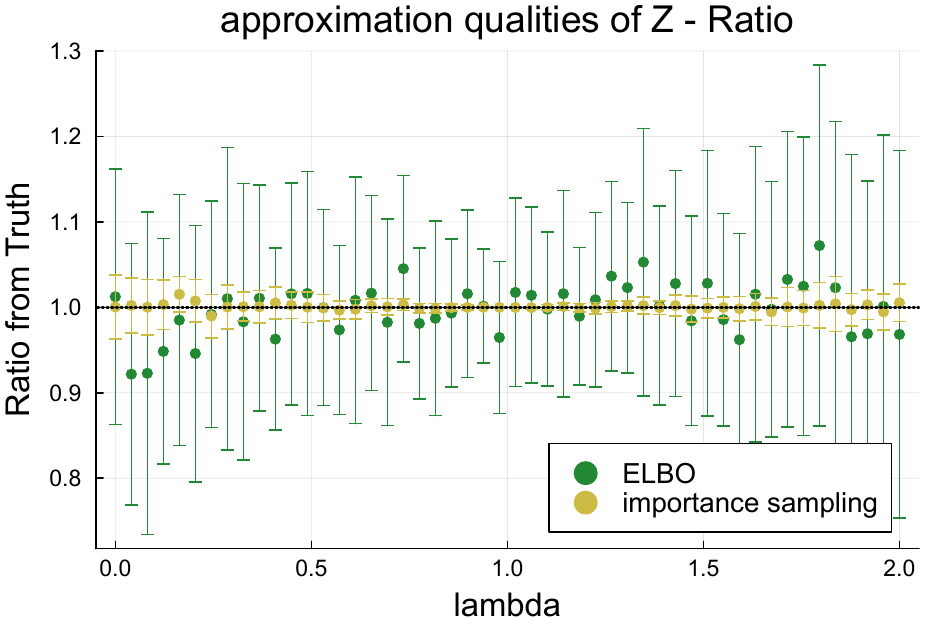}
    \caption{
    Evaluated ratios $Z^{\mathrm{approx}}_\phi(\bm{L}) / Z_\phi(\bm{L})$.
    }
    \label{fig:Z_ratios}
    \includegraphics[width = 0.75\columnwidth]{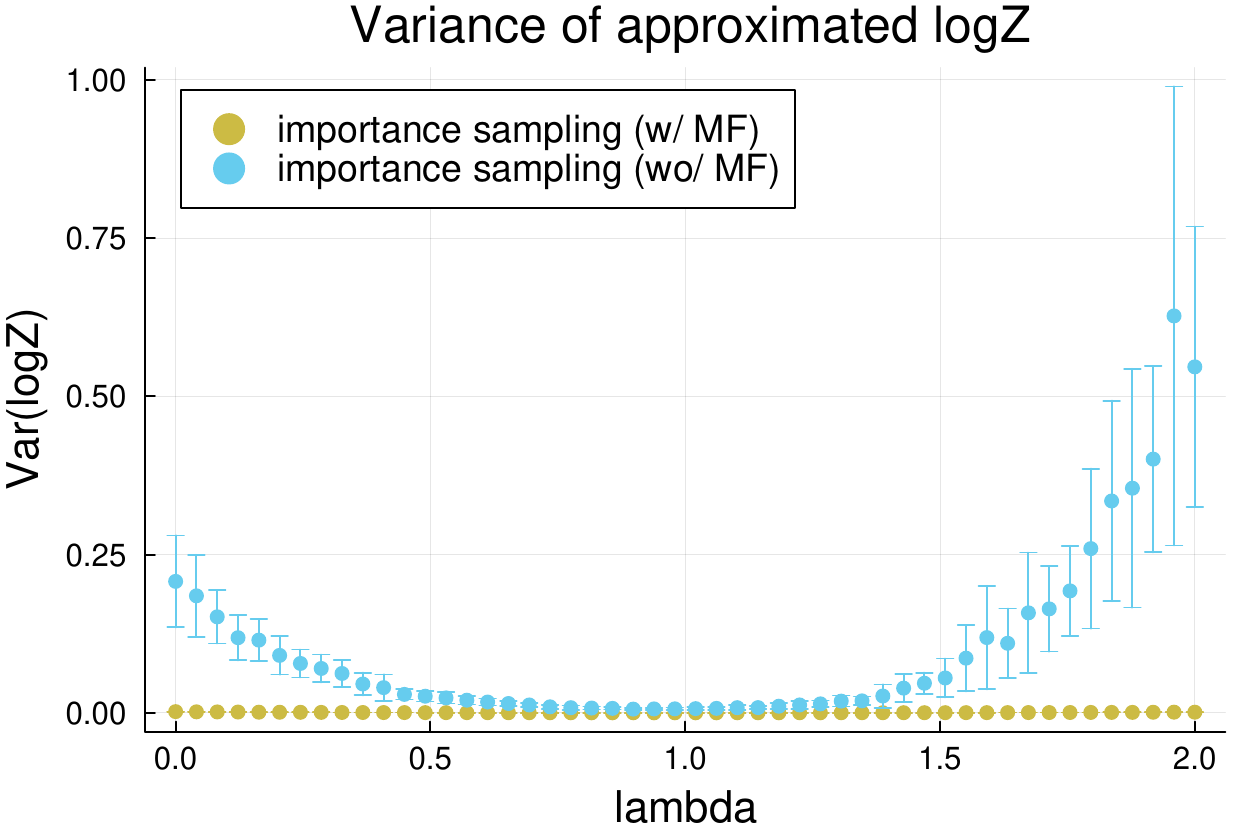}
    \caption{
    Evaluated variances of importance sampling with and without mean-field approximation.
    }
    \label{fig:Z_variance}
\end{figure}

First, we set $N = 16$ and obtain the kernel parameters from $\bm{L} \sim \mathrm{Wishart}(\bm{L}; N, \bm{I}) / N$ for each hyperparameter $\lambda \in [0, 2]$.
Then, we evaluate the gaps from the ground truth
$\log Z^{\mathrm{approx}}_{\phi_\lambda}(\bm{L}) - \log Z_{\phi_\lambda}(\bm{L})$
for each pair of $(\lambda, \bm{L})$.
Figure \ref{fig:Z_gaps} shows the results.
Although the ELBO and importance sampling seem to yield good estimates of the normalizing constant,
the gaps of the L-fields become larger as $\lambda$ moves away from $1$.
Figure \ref{fig:Z_ratios} shows the more detailed behavior of the ELBO and proposed importance sampling
through the ratio $Z^{\mathrm{approx}}_{\phi_\lambda}(\bm{L}) / Z_{\phi_\lambda}(\bm{L})$.
Although the ELBO constructs a theoretical lower bound of the normalizing constant,
it can take larger values than the ground truth owing to
the evaluation error of the Monte Carlo approximation for the multilinear extension.
We can see that the importance sampling achieves a better estimator than the ELBO,
and the variance of the importance sampling also increases as $\lambda$ moves away from $1$,
since the mean-field approximation for the proposal distribution becomes less accurate.

As an ablation study, we compare the importance sampling with and without the mean-field approximation for the proposal distribution.
We set $N = 64$ and evaluate $\mathrm{Var}(\log Z^{\mathrm{approx}}_{\phi_\lambda}(\bm{L}))$
for each $\lambda$.
We use $q_1 = \cdots = q_N = 0.5$ for the without-mean-field proposal.
As shown in Figure \ref{fig:Z_variance}, the proposal distributions with the mean-field approximation
significantly reduce the variance of the importance sampling.

\subsection{Learning the Kernel Parameter}
\begin{figure}[t]
    \centering
    \includegraphics[width = 0.75\columnwidth]{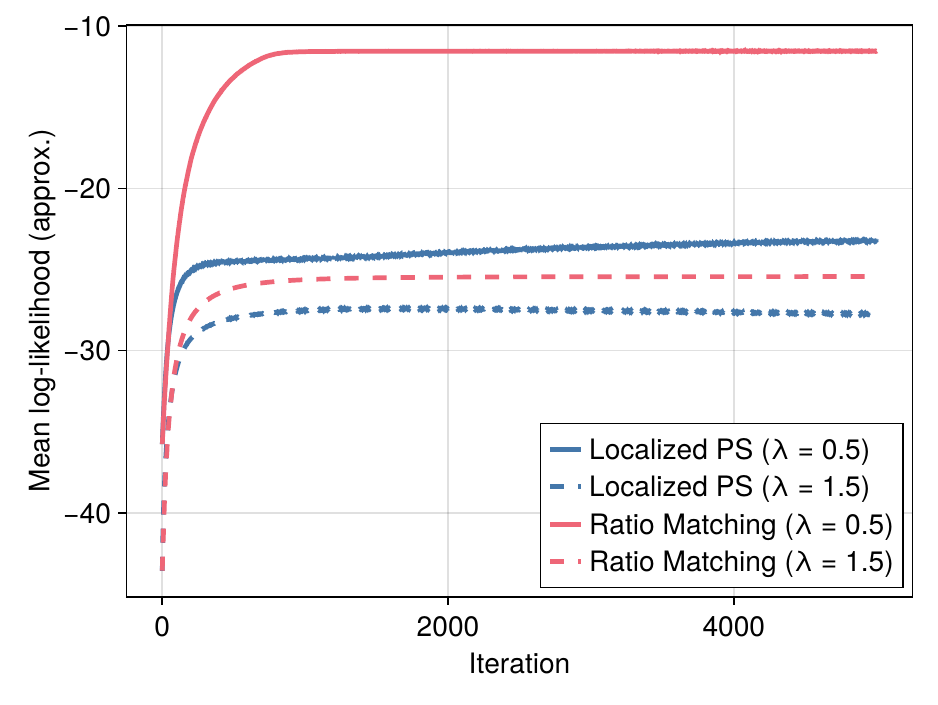}
    \caption{Learning curves with \texttt{media}. The LPSD takes $\SI{7.925 E-2}{\second}$ per iteration and ratio matching takes $\SI{2.174 E-3}{\second}$ per iteration.}
    \label{fig:rm_lpsd_amazon_media}
\end{figure}

\begin{figure}[t]
    \centering
    \includegraphics[width = 0.75\columnwidth]{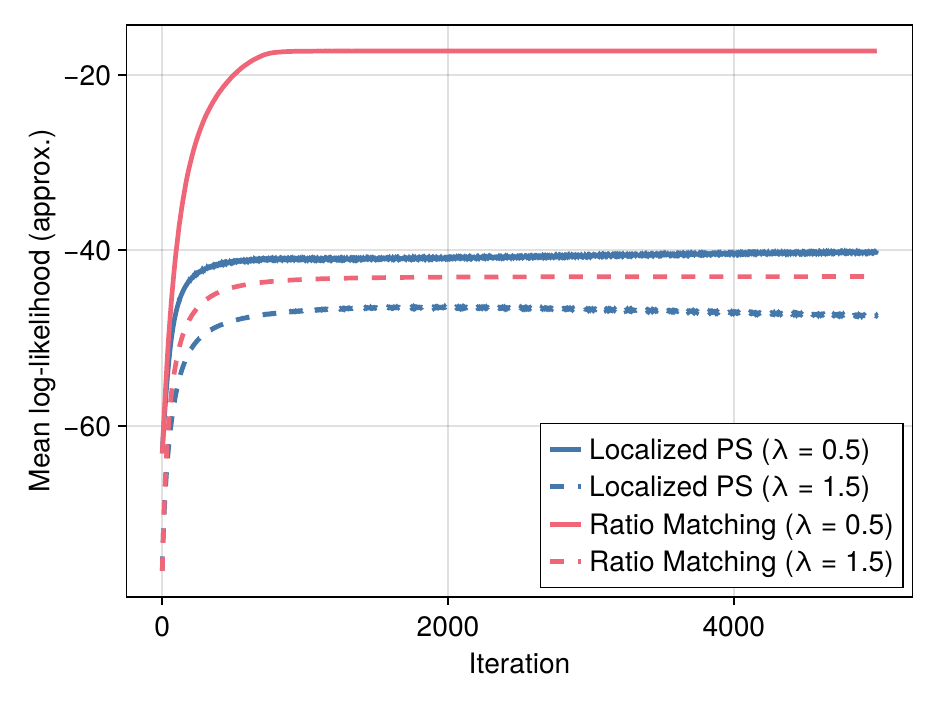}
    \caption{Learning curves with \texttt{apparel}. The LPSD takes $\SI{6.091 E-1}{\second}$ per iteration and ratio matching takes $\SI{4.155 E-3}{\second}$ per iteration.}
    \label{fig:rm_lpsd_amazon_apparel}
\end{figure}


We demonstrate kernel learning methods on the Amazon Baby Registry dataset \citep{gillenwater2014}.
The dataset contains 13 categories of childcare products,
and we use \texttt{media} ($N = 58,~ M = \num{5904}$) and \texttt{apparel} ($N = 100, M = \num{14970})$ categories.
We compare the ratio matching introduced in Section \ref{sec:learning} with the localized pseudo-spherical divergence (LPSD) \citep{takenouchi2015, takenouchi2017},
which can measure the discrepancy between an empirical distribution and
a probability function with finite support
while avoiding the calculation of the normalizing constant.
We set the Box--Cox transformation hyperparameter $\lambda = 0.5, 1.5$ and the minibatch size for ratio matching $\lvert \Omega \rvert = 100$.
Figures \ref{fig:rm_lpsd_amazon_media} and \ref{fig:rm_lpsd_amazon_apparel} show the learning curves.
The normalizing constants are evaluated approximately by importance sampling.
Overall, ratio matching performs better in both convergence and computational time.

\section{CONCLUSION}
In this study, we introduced DKPPs, a new family of distributions for random subsets
that can seamlessly control positive and negative dependence.
We developed methods for the probabilistic operations and inference in DKPPs for practical use,
including the evaluation of expectations, the normalizing constant,
and marginal and conditional probabilities.
We also proposed an efficient learning method based on ratio matching.
Empirically, we demonstrated the controllability of positive and negative dependence in DKPPs and conducted a subset acquisition experiment as the representative application.
The effectiveness of the proposed computational algorithms was also shown numerically.

Although in this paper, we mainly focused on conceptual and practical motivation,
many theoretical questions remain open.
How do DKPPs connect to point processes on continuous spaces like the marginal kernel representation of DPPs?
When do stronger conditions than log-submodularity hold as negative dependence in a DKPP?
Furthermore, we will seek an understanding of how broad the DKPP family is as a set of distributions.

\subsubsection*{Acknowledgements}
We thank Satoshi Kuriki, Keisuke Yano, Yuki Saito, Yuki Hirakawa, Takuya Furusawa, and anonymous reviewers for helpful comments.
This work was supported by JST CREST Grant Number JPMJCR2015, JSPS KAKENHI Grant Numbers JP22H03653 and 23H04483, and JST the establishment of university fellowships towards the creation of science technology innovation Grant Number JPMJFS2136.

\bibliographystyle{apalike}
\bibliography{DKPP}

\begin{thebibliography}{}

\bibitem[Ackley et~al., 1985]{ackley1985}
Ackley, D.~H., Hinton, G.~E., and Sejnowski, T.~J. (1985).
\newblock A learning algorithm for boltzmann machines.
\newblock {\em Cognitive Science}, 9(1):147--169.

\bibitem[An, 1996]{an1996}
An, M. (1996).
\newblock Log-concave {{Probability Distributions}}: {{Theory}} and
  {{Statistical Testing}}.
\newblock Game {{Theory}} and {{Information}}, University Library of Munich,
  Germany.

\bibitem[Bach, 2013]{bach2013}
Bach, F. (2013).
\newblock {\em Learning with Submodular Functions: {{A}} Convex Optimization
  Perspective}.

\bibitem[Balog et~al., 2017]{balog2017}
Balog, M., Tripuraneni, N., Ghahramani, Z., and Weller, A. (2017).
\newblock Lost {{Relatives}} of the {{Gumbel Trick}}.
\newblock In {\em Proceedings of the 34th {{International Conference}} on
  {{Machine Learning}}}, pages 371--379.

\bibitem[Bhatia, 1997]{bhatia1997}
Bhatia, R. (1997).
\newblock {\em Matrix {{Analysis}}}, volume 169 of {\em Graduate {{Texts}} in
  {{Mathematics}}}.

\bibitem[Borcea et~al., 2009]{borcea2009}
Borcea, J., Br{\"a}nd{\'e}n, P., and Liggett, T.~M. (2009).
\newblock Negative {{Dependence}} and the {{Geometry}} of {{Polynomials}}.
\newblock {\em Journal of the American Mathematical Society}, 22(2):521--567.

\bibitem[Borodin and Rains, 2005]{borodin2005}
Borodin, A. and Rains, E.~M. (2005).
\newblock Eynard--{{Mehta Theorem}}, {{Schur Process}}, and {{Their Pfaffian
  Analogs}}.
\newblock {\em Journal of Statistical Physics}, 121(3):291--317.

\bibitem[Borzadaran and Borzadaran, 2011]{borzadaran2011}
Borzadaran, G.~M. and Borzadaran, H.~M. (2011).
\newblock Log-concavity property for some well-known distributions.
\newblock {\em Surveys in Mathematics and its Applications}, 6:203--219.

\bibitem[Buchbinder et~al., 2012]{buchbinder2012}
Buchbinder, N., Feldman, M., Naor, J., and Schwartz, R. (2012).
\newblock A {{Tight Linear Time}} (1/2)-{{Approximation}} for {{Unconstrained
  Submodular Maximization}}.
\newblock In {\em 2012 {{IEEE}} 53rd {{Annual Symposium}} on {{Foundations}} of
  {{Computer Science}}}, pages 649--658.

\bibitem[Buchbinder et~al., 2014]{buchbinder2014}
Buchbinder, N., Feldman, M., Naor, J.~S., and Schwartz, R. (2014).
\newblock Submodular {{Maximization}} with {{Cardinality Constraints}}.
\newblock In {\em Proceedings of the {{Twenty-Fifth Annual ACM-SIAM Symposium}}
  on {{Discrete Algorithms}}}, pages 1433--1452.

\bibitem[Calinescu et~al., 2007]{calinescu2007}
Calinescu, G., Chekuri, C., P{\'a}l, M., and Vondr{\'a}k, J. (2007).
\newblock Maximizing a {{Submodular Set Function Subject}} to a {{Matroid
  Constraint}} ({{Extended Abstract}}).
\newblock In {\em Integer {{Programming}} and {{Combinatorial Optimization}}},
  pages 182--196.

\bibitem[Casella and Robert, 1996]{casella1996}
Casella, G. and Robert, C.~P. (1996).
\newblock Rao-{{Blackwellisation}} of {{Sampling Schemes}}.
\newblock {\em Biometrika}, 83(1):81--94.

\bibitem[Chakrabarty et~al., 2014]{chakrabarty2014}
Chakrabarty, D., Jain, P., and Kothari, P. (2014).
\newblock Provable {{Submodular Minimization}} using {{Wolfe}}' s
  {{Algorithm}}.
\newblock In {\em Advances in {{Neural Information Processing Systems}}},
  volume~27.

\bibitem[Chekuri et~al., 2014]{chekuri2014}
Chekuri, C., Vondr{\'a}k, J., and Zenklusen, R. (2014).
\newblock Submodular {{Function Maximization}} via the {{Multilinear
  Relaxation}} and {{Contention Resolution Schemes}}.
\newblock {\em SIAM Journal on Computing}, 43(6):1831--1879.

\bibitem[Derezinski et~al., 2020a]{derezinski2020}
Derezinski, M., Liang, F., and Mahoney, M. (2020a).
\newblock Bayesian {{Experimental Design Using Regularized Determinantal Point
  Processes}}.
\newblock In {\em Proceedings of the {{Twenty Third International Conference}}
  on {{Artificial Intelligence}} and {{Statistics}}}, pages 3197--3207.

\bibitem[Derezinski et~al., 2020b]{derezinski2020a}
Derezinski, M., Liang, F.~T., Liao, Z., and Mahoney, M.~W. (2020b).
\newblock Precise expressions for random projections: {{Low-rank}}
  approximation and randomized {{Newton}}.
\newblock In {\em Advances in {{Neural Information Processing Systems}}},
  volume~33, pages 18272--18283.

\bibitem[Derezi{\'n}ski and Mahoney, 2021]{derezinski2021}
Derezi{\'n}ski, M. and Mahoney, M.~W. (2021).
\newblock Determinantal {{Point Processes}} in {{Randomized Numerical Linear
  Algebra}}.
\newblock {\em Notices of the American Mathematical Society}, 68(01):1.

\bibitem[Derezi{\'n}ski et~al., 2022]{derezinski2022}
Derezi{\'n}ski, M., Warmuth, M.~K., and Hsu, D. (2022).
\newblock Unbiased {{Estimators}} for {{Random Design Regression}}.
\newblock {\em Journal of Machine Learning Research}, 23(167):1--46.

\bibitem[Diaconis and Evans, 2000]{diaconis2000}
Diaconis, P. and Evans, S.~N. (2000).
\newblock Immanants and {{Finite Point Processes}}.
\newblock {\em Journal of Combinatorial Theory, Series A}, 91(1):305--321.

\bibitem[Djolonga et~al., 2018]{djolonga2018}
Djolonga, J., Jegelka, S., and Krause, A. (2018).
\newblock Provable {{Variational Inference}} for {{Constrained Log-Submodular
  Models}}.
\newblock In {\em Advances in {{Neural Information Processing Systems}}},
  volume~31.

\bibitem[Djolonga and Krause, 2014]{djolonga2014}
Djolonga, J. and Krause, A. (2014).
\newblock From {{MAP}} to {{Marginals}}: {{Variational Inference}} in
  {{Bayesian Submodular Models}}.
\newblock In {\em Advances in {{Neural Information Processing Systems}}},
  volume~27.

\bibitem[Doucet et~al., 2000]{doucet2000}
Doucet, A., {de Freitas}, N., Murphy, K.~P., and Russell, S.~J. (2000).
\newblock Rao-blackwellised particle filtering for dynamic bayesian networks.
\newblock In {\em Conference on Uncertainty in Artificial Intelligence}.

\bibitem[Dupuy and Bach, 2018]{dupuy2018}
Dupuy, C. and Bach, F. (2018).
\newblock Learning {{Determinantal Point Processes}} in {{Sublinear Time}}.
\newblock In {\em Proceedings of the {{Twenty-First International Conference}}
  on {{Artificial Intelligence}} and {{Statistics}}}, pages 244--257.

\bibitem[Feige et~al., 2001]{feige2001}
Feige, U., Peleg, D., and Kortsarz, G. (2001).
\newblock The {{Dense}} k -{{Subgraph Problem}}.
\newblock {\em Algorithmica}, 29(3):410--421.

\bibitem[Fortuin et~al., 1971]{fortuin1971}
Fortuin, C.~M., Ginibre, J., and Kasteleyn, P.~W. (1971).
\newblock Correlation inequalities on some partially ordered sets.
\newblock {\em Communications in Mathematical Physics}, 22(2):89--103.

\bibitem[Friedland and Gaubert, 2013]{friedland2013}
Friedland, S. and Gaubert, S. (2013).
\newblock Submodular spectral functions of principal submatrices of a hermitian
  matrix, extensions and applications.
\newblock {\em Linear Algebra and its Applications}, 438(10):3872--3884.

\bibitem[Fujishige et~al., 2006]{fujishige2006}
Fujishige, S., Hayashi, T., and Isotani, S. (2006).
\newblock The {{Minimum-Norm-Point Algorithm Applied}} to {{Submodular Function
  Minimization}} and {{Linear Programming}}.

\bibitem[Fujishige and Isotani, 2011]{fujishige2011}
Fujishige, S. and Isotani, S. (2011).
\newblock A submodular function minimization algorithm based on the
  minimum-norm base.
\newblock {\em Pacific Journal of Optimization}, 7(1):3--17.

\bibitem[Garey and Johnson, 1979]{garey1979}
Garey, M.~R. and Johnson, D.~S. (1979).
\newblock {\em {Computers and intractability : a guide to the theory of
  NP-completeness}}.

\bibitem[Gartrell et~al., 2017]{gartrell2017}
Gartrell, M., Paquet, U., and Koenigstein, N. (2017).
\newblock Low-{{Rank Factorization}} of {{Determinantal Point Processes}}.
\newblock {\em Proceedings of the AAAI Conference on Artificial Intelligence},
  31(1).

\bibitem[Gillenwater et~al., 2012]{gillenwater2012a}
Gillenwater, J., Kulesza, A., and Taskar, B. (2012).
\newblock Near-{{Optimal MAP Inference}} for {{Determinantal Point Processes}}.
\newblock In {\em Advances in {{Neural Information Processing Systems}}},
  volume~25.

\bibitem[Gillenwater et~al., 2014]{gillenwater2014}
Gillenwater, J.~A., Kulesza, A., Fox, E., and Taskar, B. (2014).
\newblock Expectation-{{Maximization}} for {{Learning Determinantal Point
  Processes}}.
\newblock In {\em Advances in {{Neural Information Processing Systems}}},
  volume~27.

\bibitem[Gotovos et~al., 2015]{gotovos2015}
Gotovos, A., Hassani, H., and Krause, A. (2015).
\newblock Sampling from {{Probabilistic Submodular Models}}.
\newblock In {\em Advances in {{Neural Information Processing Systems}}},
  volume~28.

\bibitem[Hazan and Jaakkola, 2012]{hazan2012}
Hazan, T. and Jaakkola, T.~S. (2012).
\newblock On the partition function and random maximum a-posteriori
  perturbations.
\newblock In {\em Proceedings of the 29th International Conference on Machine
  Learning, {{ICML}} 2012}.

\bibitem[Hyv{\"a}rinen, 2005]{hyvarinen2005}
Hyv{\"a}rinen, A. (2005).
\newblock Estimation of {{Non-Normalized Statistical Models}} by {{Score
  Matching}}.
\newblock {\em Journal of Machine Learning Research}, 6(24):695--709.

\bibitem[Hyv{\"a}rinen, 2007]{hyvarinen2007}
Hyv{\"a}rinen, A. (2007).
\newblock Some extensions of score matching.
\newblock {\em Computational Statistics \& Data Analysis}, 51(5):2499--2512.

\bibitem[Iwata et~al., 2001]{iwata2001}
Iwata, S., Fleischer, L., and Fujishige, S. (2001).
\newblock A combinatorial strongly polynomial algorithm for minimizing
  submodular functions.
\newblock {\em Journal of the ACM}, 48(4):761--777.

\bibitem[Iyer and Bilmes, 2015]{iyer2015a}
Iyer, R. and Bilmes, J. (2015).
\newblock Submodular {{Point Processes}} with {{Applications}} to {{Machine}}
  learning.
\newblock In {\em Proceedings of the {{Eighteenth International Conference}} on
  {{Artificial Intelligence}} and {{Statistics}}}, pages 388--397.

\bibitem[Iyer et~al., 2013]{iyer2013}
Iyer, R., Jegelka, S., and Bilmes, J. (2013).
\newblock Fast semidifferential-based submodular function optimization.
\newblock In {\em Proceedings of the 30th {{International Conference}} on
  {{International Conference}} on {{Machine Learning}} - {{Volume}} 28},
  {{ICML}}'13, pages III--855--III--863.

\bibitem[Kawashima and Hino, 2023]{kawashima2023}
Kawashima, T. and Hino, H. (2023).
\newblock Minorization-{{Maximization}} for {{Learning Determinantal Point
  Processes}}.
\newblock {\em Transactions on Machine Learning Research}.

\bibitem[Kulesza and Taskar, 2011]{kulesza2011}
Kulesza, A. and Taskar, B. (2011).
\newblock K-{{Dpps}}: {{Fixed-Size Determinantal Point Processes}}.
\newblock In {\em International {{Conference}} on {{Machine Learning}}}.

\bibitem[Kulesza and Taskar, 2012]{kulesza2012}
Kulesza, A. and Taskar, B. (2012).
\newblock Determinantal {{Point Processes}} for {{Machine Learning}}.
\newblock {\em Foundations and Trends{\textregistered} in Machine Learning},
  5(2-3):123--286.

\bibitem[Lecun et~al., 1998]{lecun1998}
Lecun, Y., Bottou, L., Bengio, Y., and Haffner, P. (1998).
\newblock Gradient-based learning applied to document recognition.
\newblock {\em Proceedings of the IEEE}, 86(11):2278--2324.

\bibitem[Liu et~al., 2022]{liu2022}
Liu, M., Liu, H., and Ji, S. (2022).
\newblock Gradient-{{Guided Importance Sampling}} for {{Learning Binary
  Energy-Based Models}}.
\newblock In {\em The {{Eleventh International Conference}} on {{Learning
  Representations}}}.

\bibitem[Macchi, 1975]{macchi1975}
Macchi, O. (1975).
\newblock The {{Coincidence Approach}} to {{Stochastic Point Processes}}.
\newblock {\em Advances in Applied Probability}, 7(1):83--122.

\bibitem[Mariet and Sra, 2015]{mariet2015}
Mariet, Z. and Sra, S. (2015).
\newblock Fixed-{{Point Algorithms}} for {{Learning Determinantal Point
  Processes}}.
\newblock In {\em Proceedings of the 32nd {{International Conference}} on
  {{Machine Learning}}}, pages 2389--2397.

\bibitem[Mariet and Sra, 2016]{mariet2016}
Mariet, Z.~E. and Sra, S. (2016).
\newblock Kronecker {{Determinantal Point Processes}}.
\newblock In {\em Advances in {{Neural Information Processing Systems}}},
  volume~29.

\bibitem[Mariet et~al., 2018]{mariet2018}
Mariet, Z.~E., Sra, S., and Jegelka, S. (2018).
\newblock Exponentiated {{Strongly Rayleigh Distributions}}.
\newblock In {\em Advances in {{Neural Information Processing Systems}}},
  volume~31.

\bibitem[Mothilal et~al., 2020]{mothilal2020}
Mothilal, R.~K., Sharma, A., and Tan, C. (2020).
\newblock Explaining machine learning classifiers through diverse
  counterfactual explanations.
\newblock In {\em Proceedings of the 2020 {{Conference}} on {{Fairness}},
  {{Accountability}}, and {{Transparency}}}, {{FAT}}* '20, pages 607--617.

\bibitem[Nemhauser et~al., 1978]{nemhauser1978}
Nemhauser, G.~L., Wolsey, L.~A., and Fisher, M.~L. (1978).
\newblock An analysis of approximations for maximizing submodular set
  functions---{{I}}.
\newblock {\em Mathematical Programming}, 14(1):265--294.

\bibitem[Orlin, 2009]{orlin2009}
Orlin, J.~B. (2009).
\newblock A faster strongly polynomial time algorithm for submodular function
  minimization.
\newblock {\em Mathematical Programming}, 118(2):237--251.

\bibitem[Papandreou and Yuille, 2011]{papandreou2011}
Papandreou, G. and Yuille, A.~L. (2011).
\newblock Perturb-and-{{MAP}} random fields: {{Using}} discrete optimization to
  learn and sample from energy models.
\newblock In {\em 2011 {{International Conference}} on {{Computer Vision}}},
  pages 193--200.

\bibitem[Pemantle, 2000]{pemantle2000}
Pemantle, R. (2000).
\newblock Towards a theory of negative dependence.
\newblock {\em Journal of Mathematical Physics}, 41(3):1371--1390.

\bibitem[Rebeschini and Karbasi, 2015]{rebeschini2015}
Rebeschini, P. and Karbasi, A. (2015).
\newblock Fast {{Mixing}} for {{Discrete Point Processes}}.
\newblock In {\em Proceedings of {{The}} 28th {{Conference}} on {{Learning
  Theory}}}, pages 1480--1500.

\bibitem[Rubinstein and Kroese, 2008]{rubinstein2008}
Rubinstein, R.~Y. and Kroese, D.~P. (2008).
\newblock {\em Simulation and the Monte Carlo Method}.
\newblock Wiley Series in Probability and Statistics. 2nd ed edition.

\bibitem[Sch{\"o}lkopf et~al., 1997]{scholkopf1997}
Sch{\"o}lkopf, B., Smola, A., and M{\"u}ller, K.-R. (1997).
\newblock Kernel principal component analysis.
\newblock In {\em Artificial {{Neural Networks}} --- {{ICANN}}'97}, volume
  1327, pages 583--588.

\bibitem[Schrijver, 2000]{schrijver2000}
Schrijver, A. (2000).
\newblock A {{Combinatorial Algorithm Minimizing Submodular Functions}} in
  {{Strongly Polynomial Time}}.
\newblock {\em Journal of Combinatorial Theory, Series B}, 80(2):346--355.

\bibitem[Svitkina and Fleischer, 2011]{svitkina2011}
Svitkina, Z. and Fleischer, L. (2011).
\newblock Submodular {{Approximation}}: {{Sampling-based Algorithms}} and
  {{Lower Bounds}}.
\newblock {\em SIAM Journal on Computing}, 40(6):1715--1737.

\bibitem[Takenouchi and Kanamori, 2015]{takenouchi2015}
Takenouchi, T. and Kanamori, T. (2015).
\newblock Empirical {{Localization}} of {{Homogeneous Divergences}} on
  {{Discrete Sample Spaces}}.
\newblock In {\em Advances in {{Neural Information Processing Systems}}},
  volume~28.

\bibitem[Takenouchi and Kanamori, 2017]{takenouchi2017}
Takenouchi, T. and Kanamori, T. (2017).
\newblock Statistical {{Inference}} with {{Unnormalized Discrete Models}} and
  {{Localized Homogeneous Divergences}}.
\newblock {\em Journal of Machine Learning Research}, 18(56):1--26.

\bibitem[Tschiatschek et~al., 2016]{tschiatschek2016}
Tschiatschek, S., Djolonga, J., and Krause, A. (2016).
\newblock Learning {{Probabilistic Submodular Diversity Models Via Noise
  Contrastive Estimation}}.
\newblock In {\em Proceedings of the 19th {{International Conference}} on
  {{Artificial Intelligence}} and {{Statistics}}}, pages 770--779.

\bibitem[Valiant, 1979]{valiant1979}
Valiant, L.~G. (1979).
\newblock The complexity of computing the permanent.
\newblock {\em Theoretical Computer Science}, 8(2):189--201.

\bibitem[{Vere-Jones}, 1997]{vere-jones1997}
{Vere-Jones}, D. (1997).
\newblock Alpha-permanents and their applications to multivariate gamma,
  negative binomial and ordinary binomial distributions.
\newblock {\em New Zealand J. Math}, 26(1):125--149.

\bibitem[Wilhelm et~al., 2018]{wilhelm2018}
Wilhelm, M., Ramanathan, A., Bonomo, A., Jain, S., Chi, E.~H., and Gillenwater,
  J. (2018).
\newblock Practical {{Diversified Recommendations}} on {{YouTube}} with
  {{Determinantal Point Processes}}.
\newblock In {\em Proceedings of the 27th {{ACM International Conference}} on
  {{Information}} and {{Knowledge Management}}}, pages 2165--2173.

\bibitem[Zhang et~al., 2022]{zhang2022}
Zhang, R., Liu, X., and Liu, Q. (2022).
\newblock A {{Langevin-like Sampler}} for {{Discrete Distributions}}.
\newblock In {\em Proceedings of the 39th {{International Conference}} on
  {{Machine Learning}}}, pages 26375--26396.

\end{thebibliography}

     \section*{Checklist}
     
%

      \begin{enumerate}

      \item For all models and algorithms presented, check if you include:
      \begin{enumerate}
        \item A clear description of the mathematical setting, assumptions, algorithm, and/or model. [Yes]
        \item An analysis of the properties and complexity (time, space, sample size) of any algorithm. [No]
        \item (Optional) Anonymized source code, with specification of all dependencies, including external libraries. [Yes]
      \end{enumerate}

      \item For any theoretical claim, check if you include:
      \begin{enumerate}
        \item Statements of the full set of assumptions of all theoretical results. [Yes]
        \item Complete proofs of all theoretical results. [Yes]
        \item Clear explanations of any assumptions. [Yes]     
      \end{enumerate}

      \item For all figures and tables that present empirical results, check if you include:
      \begin{enumerate}
        \item The code, data, and instructions needed to reproduce the main experimental results (either in the supplemental material or as a URL). [Yes]
        \item All the training details (e.g., data splits, hyperparameters, how they were chosen). [Yes]
              \item A clear definition of the specific measure or statistics and error bars (e.g., with respect to the random seed after running experiments multiple times). [Yes]
              \item A description of the computing infrastructure used. (e.g., type of GPUs, internal cluster, or cloud provider). [Yes]
      \end{enumerate}
     
      \item If you are using existing assets (e.g., code, data, models) or curating/releasing new assets, check if you include:
      \begin{enumerate}
        \item Citations of the creator If your work uses existing assets. [Yes]
        \item The license information of the assets, if applicable. [Not Applicable]
        \item New assets either in the supplemental material or as a URL, if applicable. [Not Applicable]
        \item Information about consent from data providers/curators. [Yes]
        \item Discussion of sensible content if applicable, e.g., personally identifiable information or offensive content. [Not Applicable]
      \end{enumerate}
     
      \item If you used crowdsourcing or conducted research with human subjects, check if you include:
      \begin{enumerate}
        \item The full text of instructions given to participants and screenshots. [Not Applicable]
        \item Descriptions of potential participant risks, with links to Institutional Review Board (IRB) approvals if applicable. [Not Applicable]
        \item The estimated hourly wage paid to participants and the total amount spent on participant compensation. [Not Applicable]
      \end{enumerate}
     
      \end{enumerate}

\newpage
\appendix
\onecolumn

\renewcommand{\theequation}{\thesection.\arabic{equation}}
\setcounter{equation}{0}

\section{PROOFS}\label{app:proofs}
\subsection{Proof of Proposition \ref{prop:dkpp_bm}}
\begin{proof}
    Let $\mathcal{A} \subseteq \mathcal{Y}$ be a random subset following the DKPP $P_\phi(\mathcal{A}; \bm{L})$,
    and let $\bm{z} = (z_1, \ldots, z_N)^\top$ denote the indicator vector of $\mathcal{A}$:
    \begin{align}
        z_i =
        \left \{~
        \begin{aligned}
            0 & \quad\mbox{if}~ i \notin \mathcal{A},\\
            1 & \quad\mbox{if}~ i \in \mathcal{A},\\
        \end{aligned}
        \right .
    \end{align}
    for $i = 1, \ldots, N$.
    Then, we obtain
    \begingroup
    \allowdisplaybreaks
    \begin{align}
        P_\phi(\mathcal{A}; \bm{L}) &\propto \exp \tr \phi(\bm{L}[\mathcal{A}])\\
        &= \exp \tr \left ( a \bm{L}[\mathcal{A}]\bm{L}[\mathcal{A}]  + b \bm{L}[\mathcal{A}] + c\bm{I} \right )\\
        &= \exp \left ( a \tr \bm{L}[\mathcal{A}]\bm{L}[\mathcal{A}] + b \tr \bm{L}[\mathcal{A}] + c \lvert \mathcal{A} \rvert \right )\\
        &= \exp \left ( a \lVert \bm{L}[\mathcal{A}] \rVert^2_F + b \tr \bm{L}[\mathcal{A}] + c \lvert \mathcal{A} \rvert \right )\\
        &= \exp \left ( a \sum_{i,j \in \mathcal{A}} \lvert L_{ij} \rvert^2 + \sum_{i \in \mathcal{A}} (b L_{ii} + c) \right )\\
        &= \exp \left ( a \sum^N_{i,j = 1} \lvert L_{ij} \rvert^2 z_i z_j + \sum^N_{i = 1} (b L_{ii} + c) z_i \right )\\
        &= \exp \left ( a \sum_{i \neq j} \lvert L_{ij} \rvert^2 z_i z_j + a\sum^N_{i=1} L^2_{ii} z^2_i + \sum^N_{i = 1} (b L_{ii} + c) z_i \right )\\
        &= \exp \left ( a \sum_{i \neq j} \lvert L_{ij} \rvert^2 z_i z_j + a\sum^N_{i=1} L^2_{ii} z_i + \sum^N_{i = 1} (b L_{ii} + c) z_i  \right )\\
        \label{eq:proof1}
        &= \exp \left ( a \sum_{i \neq j} \lvert L_{ij} \rvert^2 z_i z_j + \sum^N_{i=1} (a L^2_{ii} + b L_{ii} + c) z_i \right ).
    \end{align}
    \endgroup
    By comparing \eqref{eq:proof1} and \eqref{eq:boltzmann_machine}, we prove the proposition.
\end{proof}

\subsection{Proof of Proposition \ref{prop:affine_logmod}}
\begin{proof}
    \ \\
    \textsf{[$\phi$ is affine$~\Longrightarrow~$DKPP is log-modular]}\par
    We denote the diagonal matrix with the eigenvalues of $\bm{L}[\mathcal{A}]$ by $\bm{\Lambda}_\mathcal{A}$.
    By letting $\phi(x) = bx + c$ from the assumption, we find that the log-likelihood of the DKPP is
    \begin{align}
        \log P_\phi(\mathcal{A}; \bm{L}) = \tr\phi(\bm{L}[\mathcal{A}]) - \log Z_\phi(\bm{L})
        &= \tr(b \bm{\Lambda}_\mathcal{A} + c\bm{I}) - \log Z_\phi(\bm{L})\\
        &= b \tr\bm{\Lambda}_\mathcal{A} + c \lvert \mathcal{A} \rvert - \log Z_\phi(\bm{L})\\
        &= b \tr\bm{L}_\mathcal{A} + c \lvert \mathcal{A} \rvert - \log Z_\phi(\bm{L})\\
        \label{eq:dkpp_linear_loglikelihood}
        &= \sum_{i \in \mathcal{A}} (b L_{ii} + c) - \log Z_\phi(\bm{L}).
    \end{align}
    This is a modular function for $\mathcal{A} \subseteq \mathcal{Y}$.
    
    \textsf{[DKPP is log-modular $~\Longrightarrow~$ $\phi$ is affine]}\par
    The goal of this proof is show the affinity of the single-variable function $\phi: \mathbb{R}_{\geq 0} \to \mathbb{R}$.
    From the log-modularity of the DKPP, we have
    \begin{align}
        \label{eq:dkpp_logmodularity}
        \tr\phi(\bm{L}[\mathcal{S}]) + \tr\phi(\bm{L}[\mathcal{T}])
        = \tr\phi(\bm{L}[\mathcal{S} \cup \mathcal{T}]) + \tr\phi(\bm{L}[\mathcal{S} \cap \mathcal{T}])
    \end{align}
    for every $\mathcal{S}, \mathcal{T} \subseteq \mathcal{Y}$.
    By defining $g(x) \coloneqq \phi(x) - \phi(0)$,
    we can represent $\phi$ as $\phi(x) = g(x) + \phi(0)$.
    Note that $g$ is continuous because of the continuity of $\phi$, assumed in Definition \ref{def:dkpp}.
    Now, we take $\mathcal{S} = \{i\}, \mathcal{T} = \{j\}~ (i \neq j)$ and
    the kernel matrix $\bm{L}$ arbitrarily to satisfy $L_{ij} = \sqrt{L_{ii} L_{jj}} \in \mathbb{R}$.
    Since $\bm{L}$ and $\phi$ are independent, this choice of $\bm{L}$ does not impose any restriction on $\phi$.
    
    The two eigenvalues of $\bm{L}[\mathcal{S} \cup \mathcal{T}] \in \mathbb{R}^{2 \times 2}$ are given by $\lambda_1 = L_{ii} + L_{jj}$ and $\lambda_2 = 0$.
    The l.h.s. of \eqref{eq:dkpp_logmodularity} is
    \begin{align}
        \label{eq:logmodular_linear_lhs}
        \tr\phi(\bm{L}[\mathcal{S}]) + \tr\phi(\bm{L}[\mathcal{T}])
        = \phi(L_{ii}) + \phi(L_{jj}),
    \end{align}
    and the r.h.s. is
    \begin{align}
        \tr\phi(\bm{L}[\mathcal{S} \cup \mathcal{T}]) + \tr\phi(\bm{L}[\mathcal{S} \cap \mathcal{T}])
        = \tr\phi(\bm{L}[\mathcal{S} \cup \mathcal{T}])
        = \phi(\lambda_1) + \phi(\lambda_2).
    \end{align}
    Therefore, we have the identity
    \begin{align}
        \label{eq:logmodular_linear}
        \phi(L_{ii}) + \phi(L_{jj})
        = \phi(\lambda_1) + \phi(\lambda_2).
    \end{align}
    By substituting $\phi(x) = g(x) + \phi(0)$ into \eqref{eq:logmodular_linear}, it becomes
    \begin{align}
        & \{g(L_{ii}) + \phi(0)\} + \{g(L_{jj}) + \phi(0)\}
        = \{g(\lambda_1) + \phi(0)\} + \{g(\lambda_2) + \phi(0)\}\\
        &\quad \Longleftrightarrow~ g(L_{ii}) + g(L_{jj}) = g(\lambda_1) + g(\lambda_2) = g(L_{ii} + L_{jj}) + g(0) = g(L_{ii} + L_{jj}).
    \end{align}
    Because $L_{ii}$ and $L_{jj}$ are arbitrary on $\mathbb{R}_{\geq 0}$, we have
    $g(x) + g(y) = g(x + y)$ for all $x, y \in \mathbb{R}_{\geq 0}$.
    From the continuity of $g$, this means that $g$ must be linear and $\phi$ must be affine.
\end{proof}

\subsection{Proof of Proposition \ref{prop:marginal_rao_blackwell}}
\begin{proof}
We consider the partition of $\mathcal{Y}$ given by
$\mathcal{A}_{\mathrm{sub}}, \mathcal{Y} \backslash \mathcal{A}_{\mathrm{sup}}, \mathcal{A}_{\mathrm{sup}} \backslash \mathcal{A}_{\mathrm{sub}}$
and
separate the random vector $\bm{\xi} = (\xi_1, \ldots, \xi_N)^\top$ by
$
\bm{\xi}^+ \coloneqq (\xi_i)_{i \in \mathcal{A}_{\mathrm{sub}}},
\bm{\xi}^- \coloneqq (\xi_i)_{i \in \mathcal{Y} \backslash \mathcal{A}_{\mathrm{sup}}}
$, and
$\bm{\xi}^\circ \coloneqq (\xi_i)_{i \in \mathcal{A}_{\mathrm{sup}} \backslash \mathcal{A}_{\mathrm{sub}}}$.
For example, when
$\mathcal{Y} = \{1, 2, 3, 4, 5\}, \mathcal{A}_{\mathrm{sub}} = \{1, 2\}$,
and $\mathcal{A}_{\mathrm{sup}} = \{1, 2, 4, 5\} (\supseteq \mathcal{A}_{\mathrm{sub}})$,
the separation of $\bm{\xi}$ becomes
$\bm{\xi}^+ = (\xi_1, \xi_2)^\top, \bm{\xi}^- = (\xi_3)$, and $\bm{\xi}^\circ = (\xi_4, \xi_5)^\top$.
We also denote
$N^+ \coloneqq \lvert \mathcal{A}_{\mathrm{sub}} \rvert,~
N^- \coloneqq \lvert \mathcal{Y} \backslash \mathcal{A}_{\mathrm{sup}} \rvert$,
and
$N^\circ \coloneqq \lvert \mathcal{A}_{\mathrm{sup}} \backslash \mathcal{A}_{\mathrm{sub}} \rvert$
and
$\bm{\xi}^+ \sim Q_{\bm{q}^+},\bm{\xi}^- \sim Q_{\bm{q}^-}$,
and $\bm{\xi}^\circ \sim Q_{\bm{q}^\circ}$,
where
$Q_{\bm{q}^+} \coloneqq \prod_{i \in \mathcal{A}_{\mathrm{sub}}} Q_{q_i}(\xi_i),
Q_{\bm{q}^-} \coloneqq \prod_{i \in \mathcal{Y} \backslash \mathcal{A}_{\mathrm{sup}}} Q_{q_i}(\xi_i)$,
and 
$Q_{\bm{q}^\circ} \coloneqq \prod_{i \in \mathcal{A}_{\mathrm{sup}} \backslash \mathcal{A}_{\mathrm{sub}}} Q_{q_i}(\xi_i)$.
We denote the probability measure that induces $Q_{\bm{q}^+}$ by $\mathbb{Q}_{\bm{q}^+}$ and the same for $\mathbb{Q}_{\bm{q}^-}$ and $\mathbb{Q}_{\bm{q}^\circ}$.

As shown in \eqref{eq:marginal_is}, 
$
    \mathbb{E}_{\bm{\xi} \sim Q_{\bm{q}}}[w(\mathcal{A}_{\bm{\xi}}) \mathds{1} (\mathcal{A}_{\mathrm{sub}} \subseteq \mathcal{A}_{\bm{\xi}} \subseteq \mathcal{A}_{\mathrm{sup}})]
$
is equal to
$\mathbb{P}(\mathcal{A}_{\mathrm{sub}} \subseteq \mathcal{A} \subseteq \mathcal{A}_{\mathrm{sup}})$.
The tower property of expectation states that
$\mathbb{E}_X[f(X)] = \mathbb{E}_Y [\mathbb{E}_X [f(X) \vert Y]]$ generally holds
for an arbitrary pair of random variables $(X, Y)$ and an arbitrary function $f$.
By choosing $X \gets \bm{\xi}$ and
$Y \gets \bm{\xi}^+$ (and $Y \gets \bm{\xi}^-$), we have
\begin{align}
    \mathbb{P}(\mathcal{A}_{\mathrm{sub}} \subseteq \mathcal{A} \subseteq \mathcal{A}_{\mathrm{sup}})
    &= \mathbb{E}_{\bm{\xi} \sim Q_{\bm{q}}}[w(\mathcal{A}_{\bm{\xi}})
    \mathds{1} (\mathcal{A}_{\mathrm{sub}} \subseteq \mathcal{A}_{\bm{\xi}} \subseteq
    \mathcal{A}_{\mathrm{sup}})]\\
    &= \mathbb{E}_{\bm{\xi}^+}[\mathbb{E}_{\bm{\xi}^-, \bm{\xi}^\circ}[w(\mathcal{A}_{\bm{\xi}})
    \mathds{1} (\mathcal{A}_{\mathrm{sub}} \subseteq \mathcal{A}_{\bm{\xi}} \subseteq
    \mathcal{A}_{\mathrm{sup}}) \vert \bm{\xi}^+]]\\
    &= \sum_{z_1 \in \{0, 1\}} \cdots \sum_{z_{N^+} \in \{0, 1\}}
    \mathbb{Q}_{\bm{q}^+}(\xi^+_1 = z_1, \ldots, \xi^+_{N^+} = z_{N^+})\\
    &\hphantom{
        \sum_{z_1 \in \{0, 1\}} \cdots \sum_{z_{N^+} \in \{0, 1\}}
    }
    \times\mathbb{E}_{\bm{\xi}^-, \bm{\xi}^\circ}[w(\mathcal{A}_{\bm{\xi}})
    \underbrace{
    \mathds{1} (\mathcal{A}_{\mathrm{sub}} \subseteq \mathcal{A}_{\bm{\xi}} \subseteq
    \mathcal{A}_{\mathrm{sup}})}_{\mathclap{\mbox{{\footnotesize takes $0$ if not $\xi^+_1 = \cdots =  \xi^+_{N^+} = 1$}}}}
    \vert \bm{\xi}^+ = \bm{z}]\\
    &= \mathbb{Q}_{\bm{q}^+}(\xi^+_1 = 1, \ldots, \xi^+_{N^+} = 1)
    \mathbb{E}_{\bm{\xi}^-, \bm{\xi}^\circ}[w(\mathcal{A}_{\bm{\xi}})
    \mathds{1} (\mathcal{A}_{\mathrm{sub}} \subseteq \mathcal{A}_{\bm{\xi}} \subseteq
    \mathcal{A}_{\mathrm{sup}}) \vert \bm{\xi}^+ = \bm{1}]\\
    &= Q_{\bm{q}^+}(\bm{1})
    \mathbb{E}_{\bm{\xi}^-}[
    \mathbb{E}_{\bm{\xi}^\circ}[w(\mathcal{A}_{\bm{\xi}})
    \mathds{1} (\mathcal{A}_{\mathrm{sub}} \subseteq \mathcal{A}_{\bm{\xi}} \subseteq
    \mathcal{A}_{\mathrm{sup}}) \vert \bm{\xi}^+ = \bm{1}, \bm{\xi}^-]]\\
    &= Q_{\bm{q}^+}(\bm{1})
    \sum_{z_1 \in \{0, 1\}} \cdots \sum_{z_{N^-} \in \{0, 1\}}
    \mathbb{Q}_{\bm{q}^-}(\xi^-_1 = z_1, \ldots, \xi^-_{N^-} = z_{N^-})\\
    &\hphantom{
        = Q_{\bm{q}^+}(\bm{1})
        \sum_{z_1 \in \{0, 1\}} \cdots \sum_{z_{N^-} \in \{0, 1\}}
    }\times
    \mathbb{E}_{\bm{\xi}^\circ}[w(\mathcal{A}_{\bm{\xi}})
    \underbrace{
    \mathds{1} (\mathcal{A}_{\mathrm{sub}} \subseteq \mathcal{A}_{\bm{\xi}} \subseteq
    \mathcal{A}_{\mathrm{sup}})}_{\mathclap{\mbox{{\footnotesize takes $0$ if not $\xi^-_1 = \cdots = \xi^-_{N^+} = 0$}}}}
    \vert \bm{\xi}^+ = \bm{1}, \bm{\xi}^- = \bm{z}]\\
    &=Q_{\bm{q}^+}(\bm{1}) Q_{\bm{q}^-}(\bm{0})
    \mathbb{E}_{\bm{\xi}^\circ}[w(\mathcal{A}_{\bm{\xi}})
    \mathds{1} (\mathcal{A}_{\mathrm{sub}} \subseteq \mathcal{A}_{\bm{\xi}} \subseteq
    \mathcal{A}_{\mathrm{sup}})
    \vert \bm{\xi}^+ = \bm{1}, \bm{\xi}^- = \bm{0}]\\
    &=Q_{\bm{q}^+}(\bm{1}) Q_{\bm{q}^-}(\bm{0})
    \mathbb{E}_{\bm{\xi}^\circ}[w(\mathcal{A}_{\bm{\xi}})
    \vert \bm{\xi}^+ = \bm{1}, \bm{\xi}^- = \bm{0}]\\
    &=Q_{\bm{q}^+}(\bm{1}) Q_{\bm{q}^-}(\bm{0})
    \mathbb{E}_{\bm{\xi}^\circ} \left [ \left .
    \frac{P(\mathcal{A}_{\bm{\xi}})}{Q_{\bm{q}}(\bm{\xi})}
    \right \vert \bm{\xi}^+ = \bm{1}, \bm{\xi}^- = \bm{0} \right ]\\
    &= \mathbb{E}_{\bm{\xi}^\circ} \left [ \left .
    \frac{P(\mathcal{A}_{\bm{\xi}})}{Q_{\bm{q}^\circ}(\bm{\xi}^\circ)}
    \right \vert \bm{\xi}^+ = \bm{1}, \bm{\xi}^- = \bm{0} \right ],
\end{align}
which is the first half of Proposition \ref{prop:marginal_rao_blackwell}.

On the other hand, the tower property of variance states that
$
\mathrm{Var}_X[f(X)] = \mathbb{E}_Y[\mathrm{Var}_X[f(X) \vert Y]] + \mathrm{Var}_Y[\mathbb{E}_X[f(X) \vert Y]]
\leq \mathrm{Var}_Y[\mathbb{E}_X[f(X) \vert Y]].
$
This ensures the latter half of Proposition \ref{prop:marginal_rao_blackwell}.
\end{proof}

\subsection{Proof of Proposition \ref{prop:marginal_cardinality}}
\begin{proof}
    First, we define independent Bernoulli trials
    $Q_{\bm{q}}$ as in \eqref{eq:variational_distribution}
    with $q_1 = \cdots = q_N \eqqcolon q$.
    By importance sampling, we have
    \begin{align}
        \mathbb{P}(\lvert \mathcal{A} \rvert = k)
        = \sum_{\mathcal{A}: \lvert \mathcal{A} \rvert = k} P (\mathcal{A})
        = \mathbb{E}_{\mathcal{A} \sim P}[\mathds{1} (\lvert \mathcal{A} \rvert = k)]
        \label{eq:marginal_ls_cardinality_1}
        = \mathbb{E}_{\bm{\xi} \sim Q_{\bm{q}}}[w(\mathcal{A}_{\bm{\xi}}) \mathds{1} (\lvert \mathcal{A}_{\bm{\xi}} \rvert = k)]
    \end{align}
    in the similar way to \eqref{eq:marginal_is},
    where $w(\mathcal{A}_{\bm{\xi}}) = P(\mathcal{A}_{\bm{\xi}}) / Q_{\bm{q}}(\bm{\xi})$
    is the weight function.
    Then, we introduce a new random variable,
    $\zeta \coloneqq \sum^N_{i=1} \xi_i$,
    that follows the binomial distribution: $\xi \sim \mathrm{Bin}(N, q)$.
    Now, we consider the Rao--Blackwellization of \eqref{eq:marginal_ls_cardinality_1} by the auxiliary random variable $\zeta$:
    \begin{align}
        \mathbb{P}(\lvert \mathcal{A} \rvert = k)
        &= \mathbb{E}_{\bm{\xi}}[w(\mathcal{A}_{\bm{\xi}}) \mathds{1} (\lvert \mathcal{A}_{\bm{\xi}} \rvert = k)]\\
        &= \mathbb{E}_{\zeta} \left [\mathbb{E}_{\bm{\xi}} \left [w(\mathcal{A}_{\bm{\xi}}) \mathds{1} (\lvert \mathcal{A}_{\bm{\xi}} \rvert = k) \left \vert \sum^N_{i=1} \xi_i = \zeta \right . \right ] \right ]\\
        &= \sum^N_{n = 0} \mathbb{P}(\zeta = n) \mathbb{E}_{\bm{\xi}} \left [ w(\mathcal{A}_{\bm{\xi}}) 
        \underbrace{\mathds{1} (\lvert \mathcal{A}_{\bm{\xi}} \rvert = k)}_{\mathclap{\mbox{{\footnotesize takes 0 if $\zeta \neq k$}}}}
        \left \vert \sum^N_{i=1} \xi_i = n \right . \right ]\\
        &= \mathbb{P}(\zeta = k) \mathbb{E}_{\bm{\xi}}\left [w(\mathcal{A}_{\bm{\xi}}) \left \vert \sum^N_{i=1} \xi_i = k \right . \right ]\\
        &= \binom{N}{k} q^{k} (1-q)^{N-k} \mathbb{E}_{\bm{\xi}} \left [ \frac{P(\mathcal{A}_{\bm{\xi}})}{Q_{\bm{q}}(\mathcal{A}_{\bm{\xi}})} \left \vert \sum^N_{i=1} \xi_i = k \right .\right ]\\
        &= \binom{N}{k} q^{k} (1-q)^{N-k} \mathbb{E}_{\bm{\xi}} \left [ \frac{P(\mathcal{A}_{\bm{\xi}})}{q^k (1-q)^{N-k}} \left \vert \sum^N_{i=1} \xi_i = k \right .\right ]\\
        \label{eq:marginal_ls_cardinality}
        &= \binom{N}{k} \mathbb{E}_{\bm{\xi}} \left [ P(\mathcal{A}_{\bm{\xi}}) \left \vert \sum^N_{i=1} \xi_i = k \right . \right ].
    \end{align}
    Because $\xi_1, \ldots, \xi_N$ are i.i.d. such that $\xi_i \sim \mathrm{Bernoulli}(q)$,
    the conditional expectation in the r.h.s. of \eqref{eq:marginal_ls_cardinality}
    equals to the expectation over the uniform distribution on $\{\mathcal{A} \subseteq \mathcal{Y} : \lvert \mathcal{A} \rvert = k\}$.
\end{proof}

\section{MEAN-FIELD APPROXIMATION}\label{app:mean_field}
We derive the update rule of the mean-field approximation \eqref{eq:mean_field_update} for completeness.
It is recommended to also refer to the thesis by \citet[Section 3]{djolonga2018} since the derivation is equivalent.
Let $\bm{\xi} \in \{0, 1\}^N$ denote a binary random vector,
$f: 2^\mathcal{Y} \to \mathbb{R}$ be a set function,
and $P: \bm{\xi} \mapsto Z^{-1} \exp (f(\mathcal{A}_{\bm{\xi}}))$ be a probabilistic function on $\{0, 1\}^N$, or equivalently $2^\mathcal{Y}$.
Now, we consider minimizing $\mathrm{KL}(Q_{\bm{q}} \Vert P)$, where $Q_{\bm{q}}$ is defined in \eqref{eq:variational_distribution}.
Given
\begin{align}
    \mathrm{KL}(Q_{\bm{q}} \Vert P)
    = \mathbb{E}_{\bm{\xi} \sim Q_{\bm{q}}}\left [ \log \frac{Q_{\bm{q}} (\bm{\xi})}{P(\bm{\xi})} \right ]
    = \mathbb{E}_{\bm{\xi} \sim Q_{\bm{q}}} [ \log Q_{\bm{q}} (\bm{\xi}) ]
    - \mathbb{E}_{\bm{\xi} \sim Q_{\bm{q}}}[ f(\mathcal{A}_{\bm{\xi}}) ] + \log Z,
\end{align}
the minimization of $\mathrm{KL}(Q_{\bm{q}} \Vert P)$ is equivalent to maximizing the ELBO,
defined as
\begin{align}
    L(\bm{q}) \coloneqq
    -\mathbb{E}_{\bm{\xi} \sim Q_{\bm{q}}} [ \log Q_{\bm{q}} (\bm{\xi}) ]
    + \mathbb{E}_{\bm{\xi} \sim Q_{\bm{q}}}[ f(\mathcal{A}_{\bm{\xi}}) ]
    ~(= \mathbb{H}[Q_{\bm{q}}] + \tilde{f}(\bm{q})).
\end{align}
We solve the problem $\max_{\bm{q}} L(\bm{q})$ by using the coordinate ascent.
The derivative $\frac{\partial L(\bm{q})}{\partial q_i}$ is 
\begin{align}
    \frac{\partial L(\bm{q})}{\partial q_i}
    &= \frac{\partial L(\bm{q})}{\partial q_i}
    \left \{
    \sum^N_{j=1} (-q_j \log q_j - (1 - q_j) \log (1 - q_j))
    + \sum_{\mathcal{A} \subseteq \mathcal{Y}} f(\mathcal{A}) \prod_{j \in \mathcal{A}} q_j \prod_{j \notin \mathcal{A}} (1 - q_j)
    \right \}\\
    &= \log \frac{1-q_i}{q_i}
    + \sum_{\mathcal{A}: i \in \mathcal{A}} f(\mathcal{A}) \prod_{\substack{j \in \mathcal{A}\\j \neq i}} q_j \prod_{j \notin \mathcal{A}} (1 - q_j)
    - \sum_{\mathcal{A}: i \notin \mathcal{A}} f(\mathcal{A}) \prod_{j \in \mathcal{A}} q_j \prod_{\substack{j \notin \mathcal{A} \\ j \neq i}} (1 - q_j)\\
    &= \log \frac{1-q_i}{q_i}
    + \sum_{\mathcal{A} \subseteq \mathcal{Y} \backslash \{i\}} [f(\mathcal{A} \cup \{i\}) - f(\mathcal{A})]
    \prod_{j \in \mathcal{A}} q_j \prod_{j \notin \mathcal{A}} (1 - q_j)\\
    &= \log \frac{1-q_i}{q_i} + \mathbb{E}_{\bm{\xi}_{\backslash i} \sim Q_{\bm{q}_{\backslash i}}} [f(i \vert \mathcal{A}_{\bm{\xi}_{\backslash i}})].
\end{align}
By solving the equation
\begin{align}
    \log \frac{1-q_i}{q_i} + \mathbb{E}_{\bm{\xi}_{\backslash i} \sim Q_{\bm{q}_{\backslash i}}} [f(i \vert \mathcal{A}_{\bm{\xi}_{\backslash i}})] = 0,
\end{align}
we obtain the update rule \eqref{eq:mean_field_update}.
We know $\mathrm{KL}(Q_{\bm{q}} \Vert P) \geq 0$, leading to the inequality $\log Z \geq L(\bm{q})$.
Therefore, we can evaluate the tightened lower bound of $\log Z$ using the optimized $\bm{q}$.


\section{GRADIENT OF RATIO MATCHING}\label{app:gradient}
In this section, we derive the gradient of the loss function from ratio matching \eqref{eq:ratio_matching_dkpp}
in analytical form.
By defining
$u_{m, n} \coloneqq \exp(\tr \phi(\bm{L}[\mathcal{A}_m]) - \tr\phi(\bm{L}[\mathcal{A}^{\bar{n}}_m] ))$,
we obtain
\begin{align}
    \frac{\partial J(\bm{L})}{\partial\bm{L}}
    = \sum_{m,n} \frac{d g(u_{m,n})^2}{d u_{m,n}} \frac{\partial u_{m,n}}{\partial \bm{L}}
    \label{eq:rm_gradient_1}
    = -2\sum_{m,n} \frac{g(u_{m,n})}{(1 + u_{m,n})^2} \frac{\partial u_{m,n}}{\partial \bm{L}}.
\end{align}
Here, $\bm{U}_\mathcal{A}$ denotes the $N \times \lvert \mathcal{A} \rvert$ binary matrix
such that $\bm{L}[\mathcal{A}] = \bm{U}^\top_\mathcal{A} \bm{L} \bm{U}_\mathcal{A}$.
Then, 
\begin{align}
    \frac{\partial}{\partial \bm{L}} \tr \phi(\bm{L}[\mathcal{A}])
    = \bm{U}^\top_\mathcal{A} \phi' (\bm{L}[\mathcal{A}]) \bm{U}_\mathcal{A},
\end{align}
where $\phi'$ is the derivative of the univariate scalar function $\phi$.
Therefore, the remaining term in \eqref{eq:rm_gradient_1} becomes
\begin{align}
   \frac{\partial u_{m,n}}{\partial \bm{L}}
   = u_{m,n} (\bm{U}^\top_{\mathcal{A}_m} \phi' (\bm{L}[\mathcal{A}_m]) \bm{U}_{\mathcal{A}_m}
   - \bm{U}^\top_{\mathcal{A}^{\bar{n}}_m} \phi' (\bm{L}[\mathcal{A}^{\bar{n}}_m]) \bm{U}_{\mathcal{A}^{\bar{n}}_m}).
\end{align}
Consequently, the derivative we seek is
\begin{align}
    \label{eq:rm_gradient_L}
    \frac{\partial J(\bm{L})}{\partial\bm{L}}
    = -2\sum_{m,n} \frac{u_{m,n} g(u_{m,n})}{(1 + u_{m,n})^2}
    (\bm{U}^\top_{\mathcal{A}_m} \phi' (\bm{L}[\mathcal{A}_m]) \bm{U}_{\mathcal{A}_m}
    - \bm{U}^\top_{\mathcal{A}^{\bar{n}}_m} \phi' (\bm{L}[\mathcal{A}^{\bar{n}}_m]) \bm{U}_{\mathcal{A}^{\bar{n}}_m}).
\end{align}
For evaluating the gradient \eqref{eq:rm_gradient_L},
computing $u_{m,n}$ requires $\mathcal{O}(\lvert \mathcal{A}_m \rvert^3)$ time complexity,
and
$\bm{U}^\top_{\mathcal{A}_m} \phi' (\bm{L}[\mathcal{A}_m]) \bm{U}_{\mathcal{A}_m}$
also takes $\mathcal{O}(\lvert \mathcal{A}_m \rvert^3)$
because $\bm{U}_{\mathcal{A}_m}$ has only $\lvert \mathcal{A}_m \rvert$ non-zero elements.
Computing
$\bm{U}^\top_{\mathcal{A}^{\bar{n}}_m} \phi' (\bm{L}[\mathcal{A}^{\bar{n}}_m]) \bm{U}_{\mathcal{A}^{\bar{n}}_m}$
takes $\mathcal{O}((\lvert \mathcal{A}_m \rvert + 1)^3) =\mathcal{O}(\lvert \mathcal{A}_m \rvert^3)$.
By taking the complexity from the sum of $N \times N$ matrices into account,
we obtain the whole complexity
$\mathcal{O}(\sum_{(m,n) \in \Omega} \lvert \mathcal{A}_m \rvert^3 + \lvert \Omega \rvert N^2)
= \mathcal{O}(\lvert \Omega \rvert (\kappa^3 + N^2))$
with the minibatch $\Omega$.

In practical scenarios, $\bm{V} \in \mathbb{R}^{N \times D}$ such that $\bm{L} = \bm{V} \bm{V}^\top$
is often learned to keep $\bm{L}$ positive (semi-)definite in learning steps.
If $D < N$, the low-rank kernel matrix is obtained.
Then, the gradient with respect to $\bm{V}$ becomes
\begin{align}
    \frac{\partial J(\bm{L})}{\partial\bm{V}}
    &= 2\frac{\partial J(\bm{L})}{\partial\bm{L}} \bm{V}\\
    \label{eq:rm_gradient_V}
    &= -4\sum_{m,n} \frac{u_{m,n} g(u_{m,n})}{(1 + u_{m,n})^2}
    (\bm{U}^\top_{\mathcal{A}_m} \phi' (\bm{L}[\mathcal{A}_m]) \bm{U}_{\mathcal{A}_m} \bm{V}
    - \bm{U}^\top_{\mathcal{A}^{\bar{n}}_m} \phi' (\bm{L}[\mathcal{A}^{\bar{n}}_m]) \bm{U}_{\mathcal{A}^{\bar{n}}_m} \bm{V}).
\end{align}
Since $\bm{U}_{\mathcal{A}_m} \bm{V}$ is the $\lvert \mathcal{A}_m \rvert \times D$ dense matrix,
the time complexity of \eqref{eq:rm_gradient_V} is
$\mathcal{O}(\sum_{(m,n) \in \Omega}(\lvert \mathcal{A}_m \rvert^3 + D\lvert \mathcal{A}_m \rvert^2) + \lvert \Omega \rvert ND)
= \mathcal{O}(\lvert \Omega \rvert (\kappa^2 \max\{\kappa, D\} + ND))$.
The term $\mathcal{O}(\lvert \Omega \rvert ND)$ arises from matrix additions,
which still ensures scalability as $M$ and/or $N$ increases even if $D = N$.

\section{FURTHER EXPERIMENTS}\label{app:experiments}
\subsection{Subset Acquiring Experiment}
\begin{figure}[t]
    \centering
    \includegraphics[width = 0.75\columnwidth]{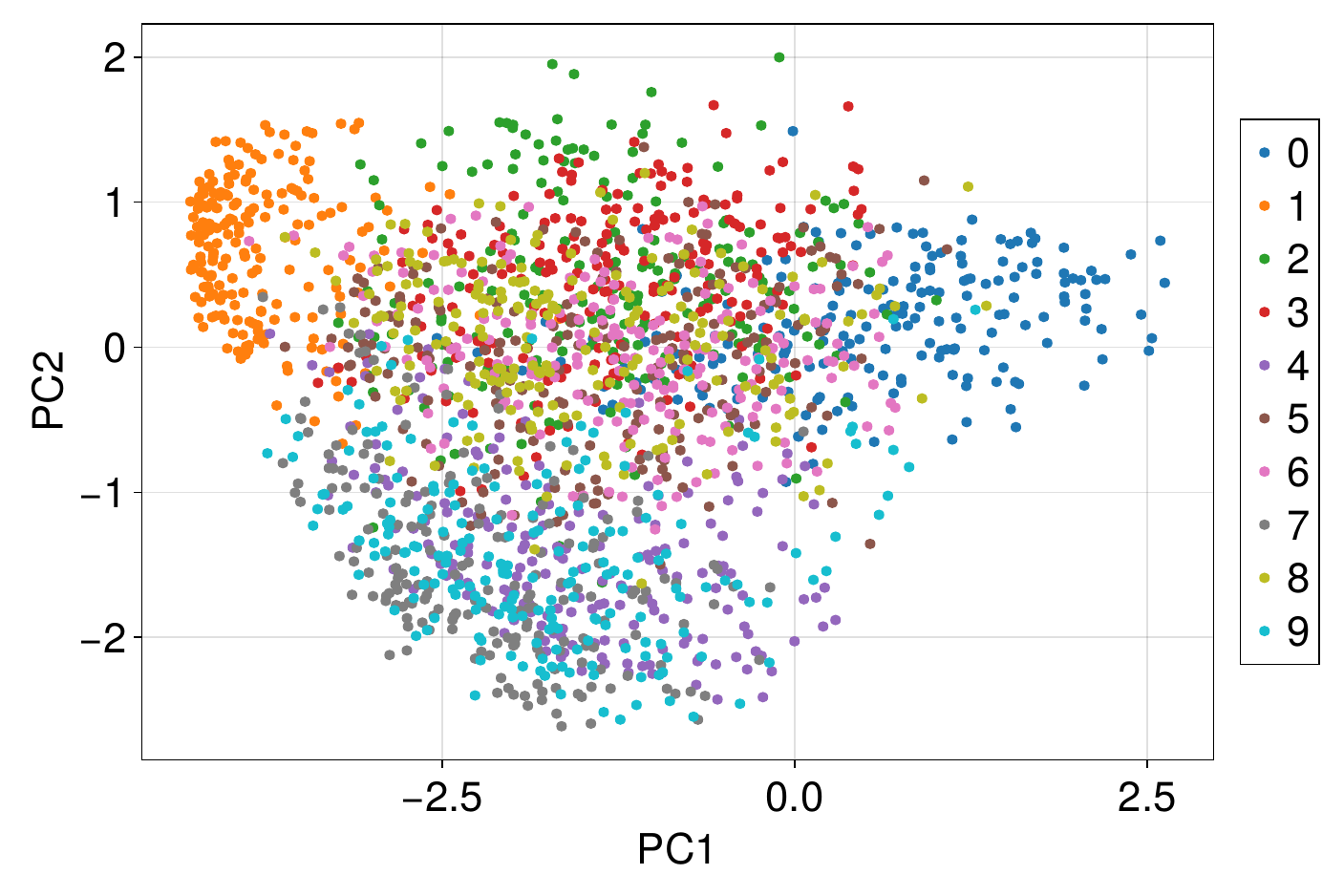}
    \caption{Kernel PCA for MNIST.}
    \label{fig:mnist_kernelpca}
\end{figure}

Here, we show additional results of the subset acquiring experiment in Section \ref{sec:experiments}.
As stated in Section \ref{sec:experiments}, the kernel matrix $\bm{L}$ is constructed from the Gaussian kernel with the bandwidth parameter determined by the median heuristic.
We can make the centered kernel matrix
$\tilde{\bm{L}} \coloneqq (\bm{I} - \bm{1}\bm{1}^\top / N) \bm{L}(\bm{I} - \bm{1}\bm{1}^\top / N)$
and apply kernel principal component analysis (PCA) to $\tilde{\bm{L}}$ \citep{scholkopf1997}.
Figure \ref{fig:mnist_kernelpca} shows the first and second principal components of the MNIST obtained by the kernel.
We can see that the kernel has a certain capability for class separation.

In Figures \ref{fig:beta_1}, \ref{fig:beta_10}, and \ref{fig:beta_50}, we show 10 randomly chosen acquired subsets for each $(\beta, \lambda) \in \{1, 10, 50\} \times \{0, 1, 2\}$.
It is visually evident that the attractive power increases as $\lambda$ and $\beta$ become larger.

\begin{figure}[t]
    \centering
    \begin{minipage}{0.32\columnwidth}
        \centering
        \includegraphics[width = 0.8\columnwidth]{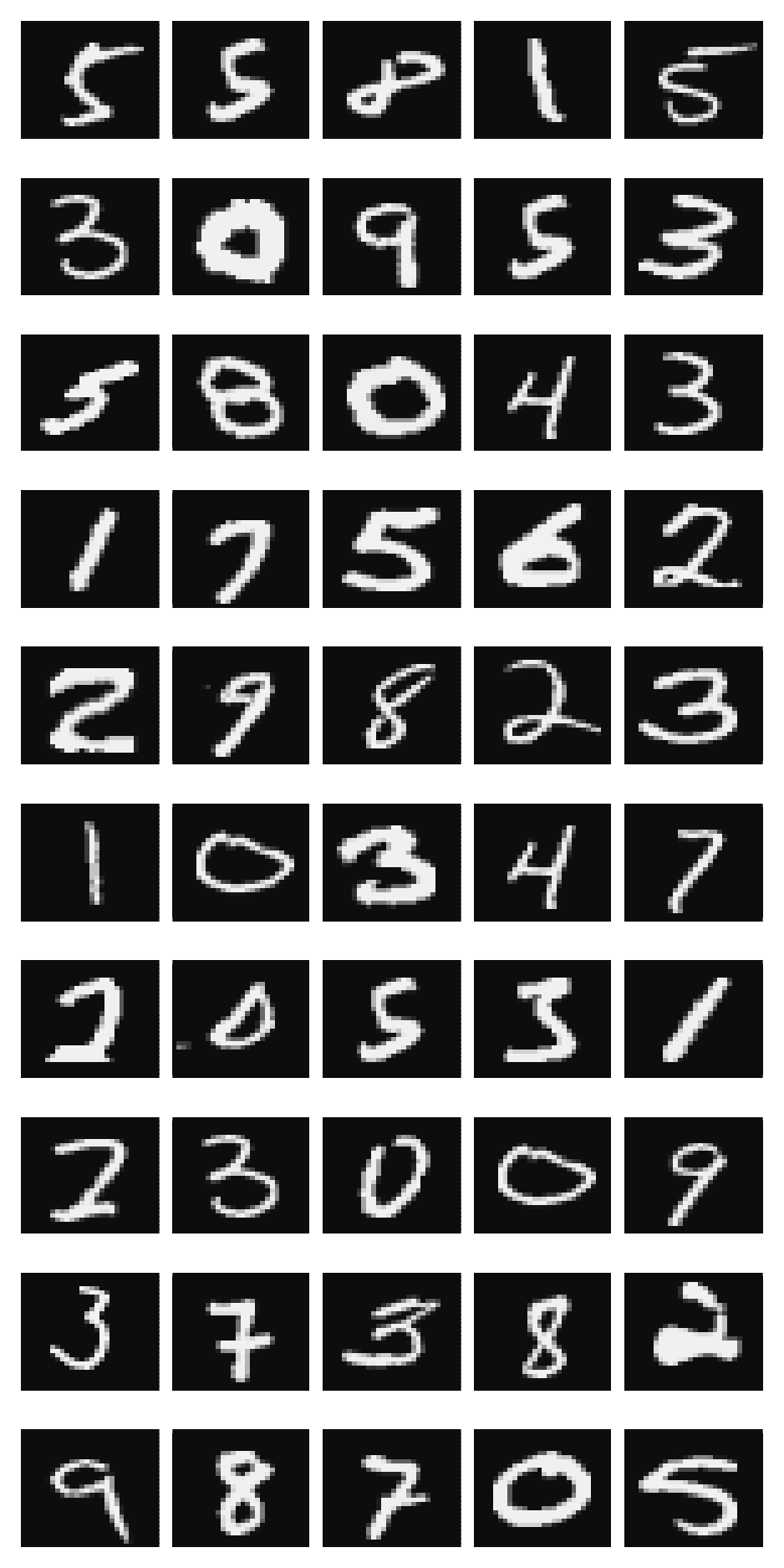}
        \subcaption{$\lambda = 0, \beta = 1$}
        \label{subfig:lambda_0_beta_1}
    \end{minipage}
    \begin{minipage}{0.32\columnwidth}
        \centering
        \includegraphics[width = 0.8\columnwidth]{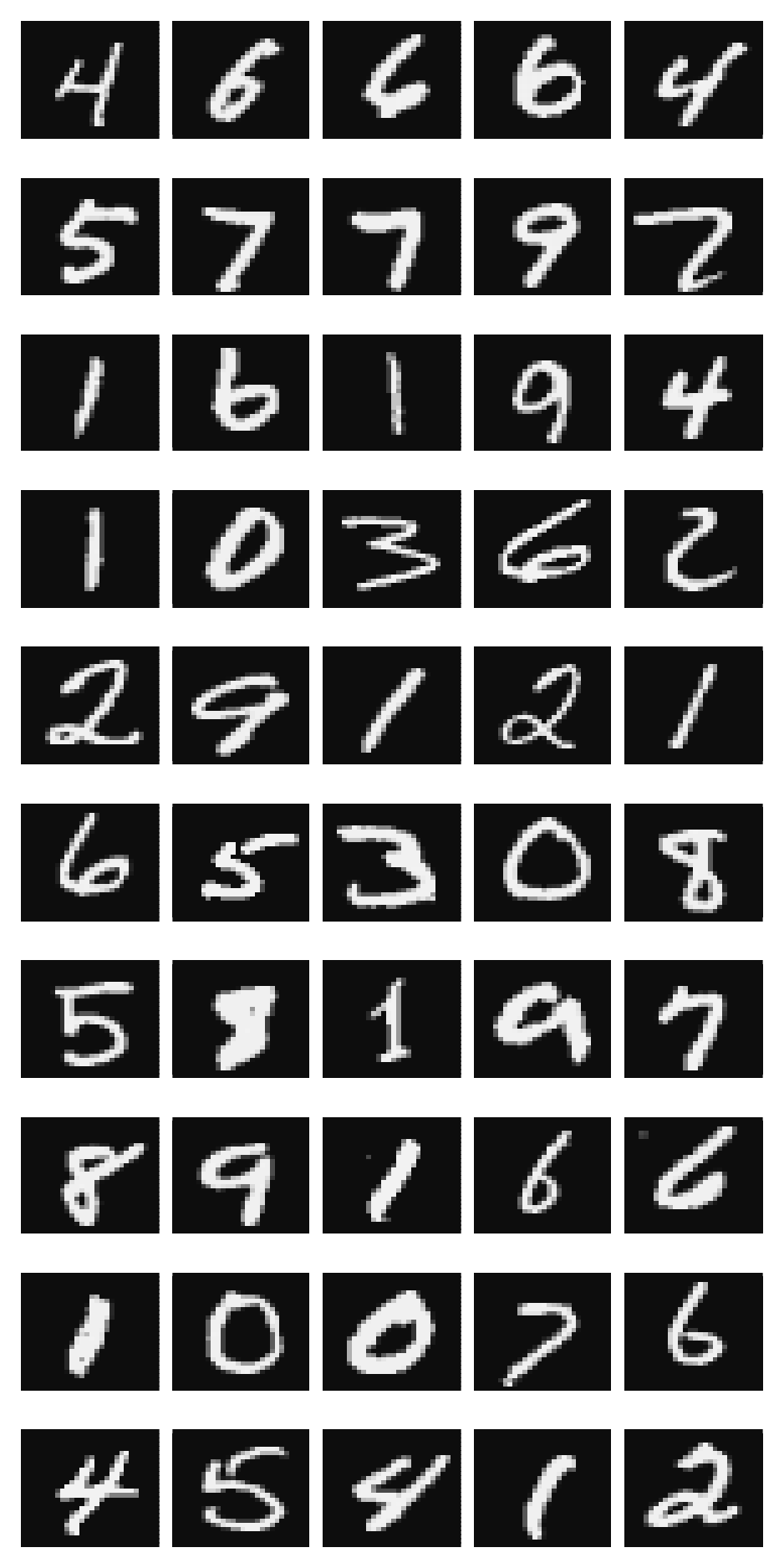}
        \subcaption{$\lambda = 1, \beta = 1$}
        \label{subfig:lambda_1_beta_1}
    \end{minipage}
    \begin{minipage}{0.32\columnwidth}
        \centering
        \includegraphics[width = 0.8\columnwidth]{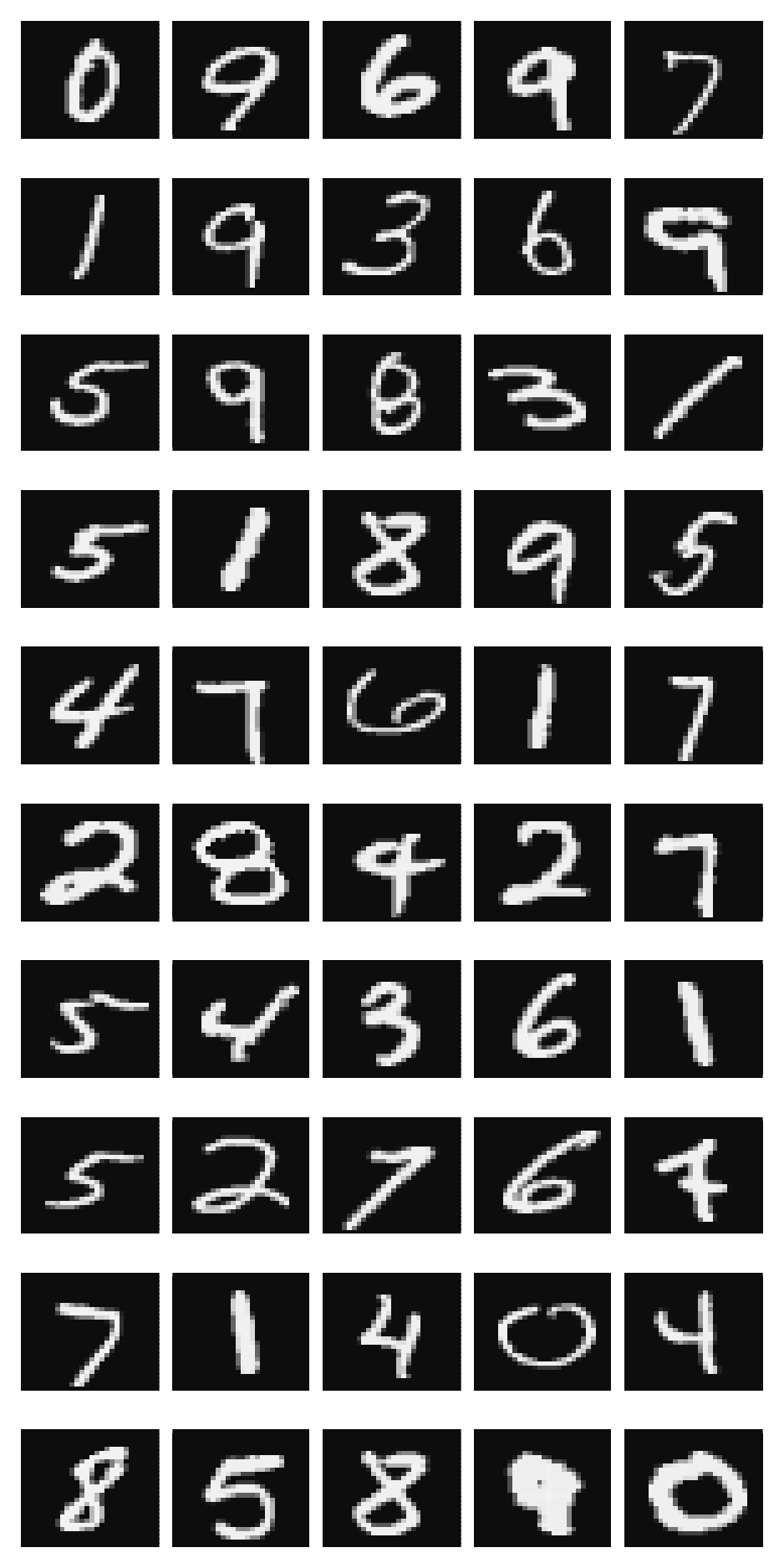}
        \subcaption{$\lambda = 2, \beta = 1$}
        \label{subfig:lambda_2_beta_1}
    \end{minipage}
    \caption{Examples of the acquired subsets of MNIST for $\beta = 1$.}
    \label{fig:beta_1}
\end{figure}

\begin{figure}[t]
    \centering
    \begin{minipage}{0.32\columnwidth}
        \centering
        \includegraphics[width = 0.8\columnwidth]{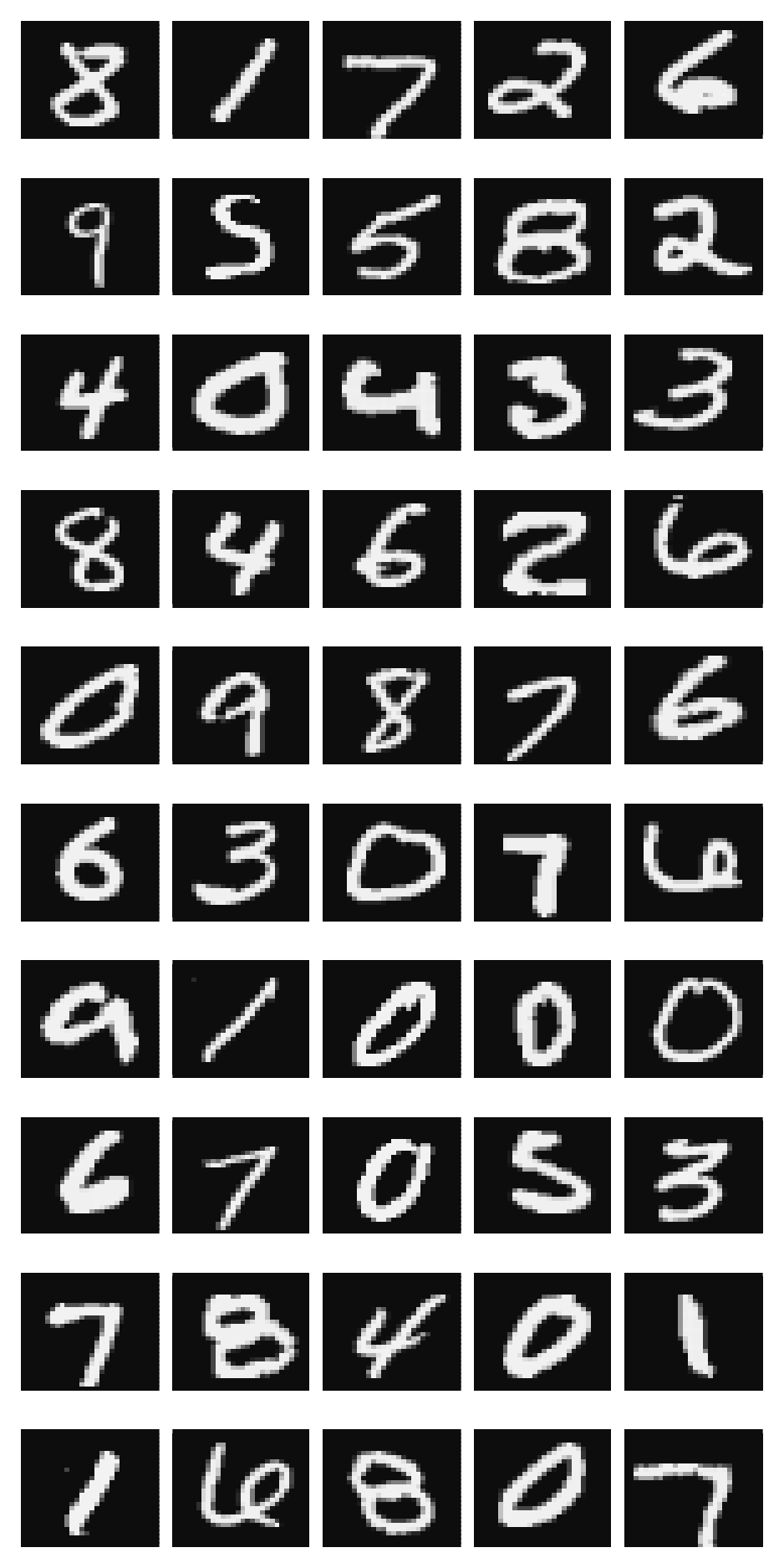}
        \subcaption{$\lambda = 0, \beta = 10$}
        \label{subfig:lambda_0_beta_10}
    \end{minipage}
    \begin{minipage}{0.32\columnwidth}
        \centering
        \includegraphics[width = 0.8\columnwidth]{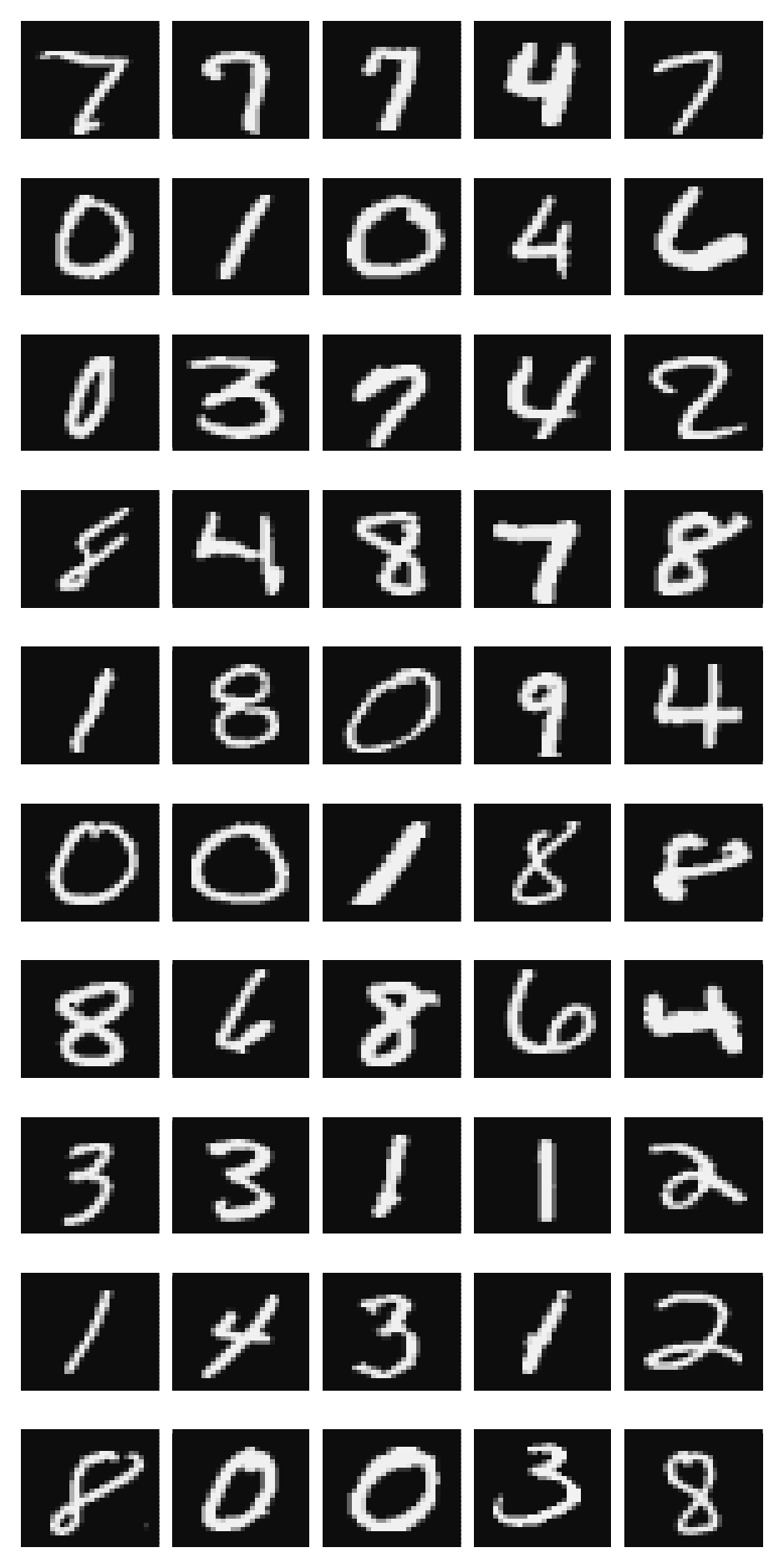}
        \subcaption{$\lambda = 1, \beta = 10$}
        \label{subfig:lambda_1_beta_10}
    \end{minipage}
    \begin{minipage}{0.32\columnwidth}
        \centering
        \includegraphics[width = 0.8\columnwidth]{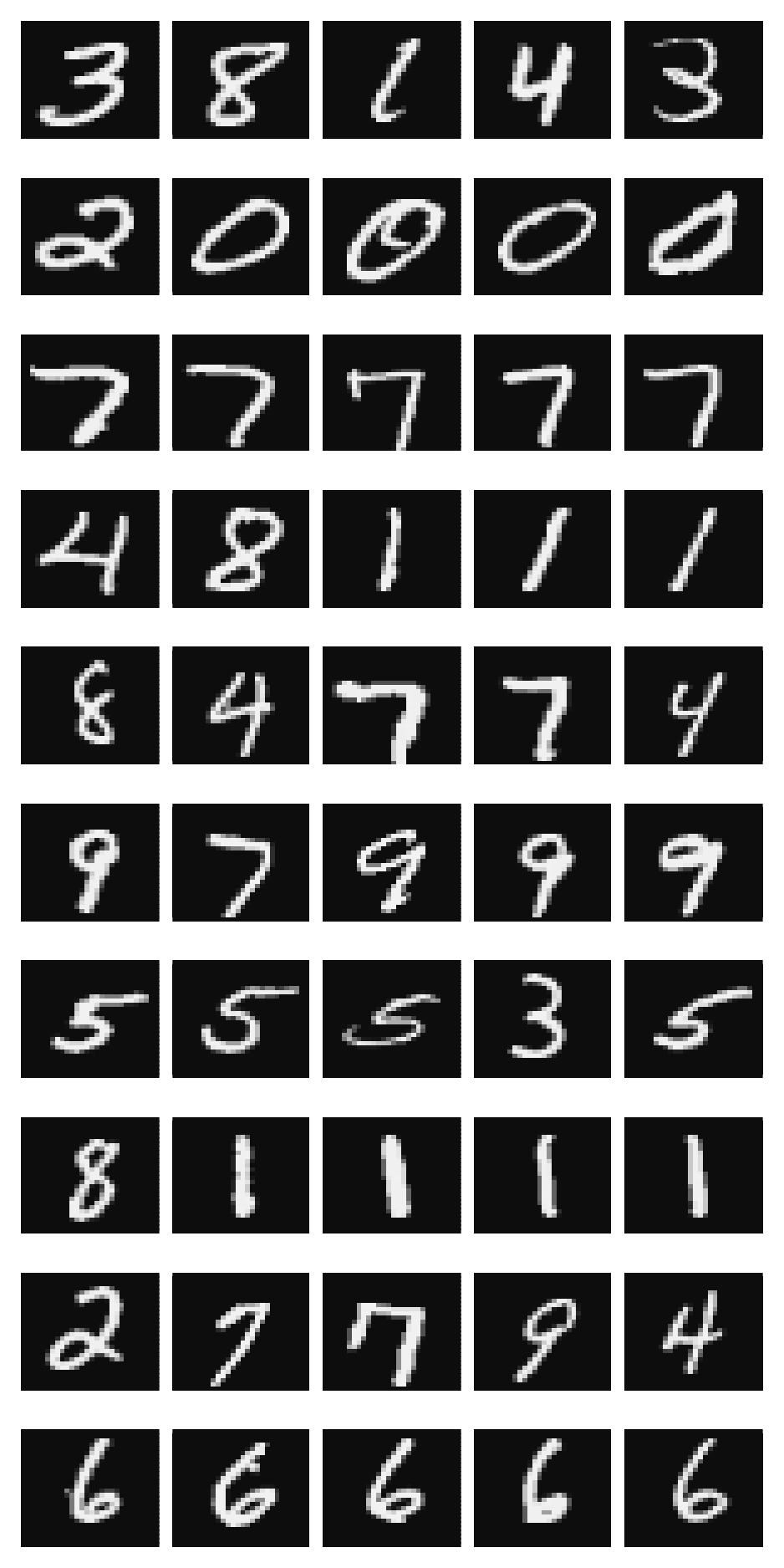}
        \subcaption{$\lambda = 2, \beta = 10$}
        \label{subfig:lambda_2_beta_10}
    \end{minipage}
    \caption{Examples of the acquired subsets of MNIST for $\beta = 10$.}
    \label{fig:beta_10}
\end{figure}

\begin{figure}[t]
    \centering
    \begin{minipage}{0.32\columnwidth}
        \centering
        \includegraphics[width = 0.8\columnwidth]{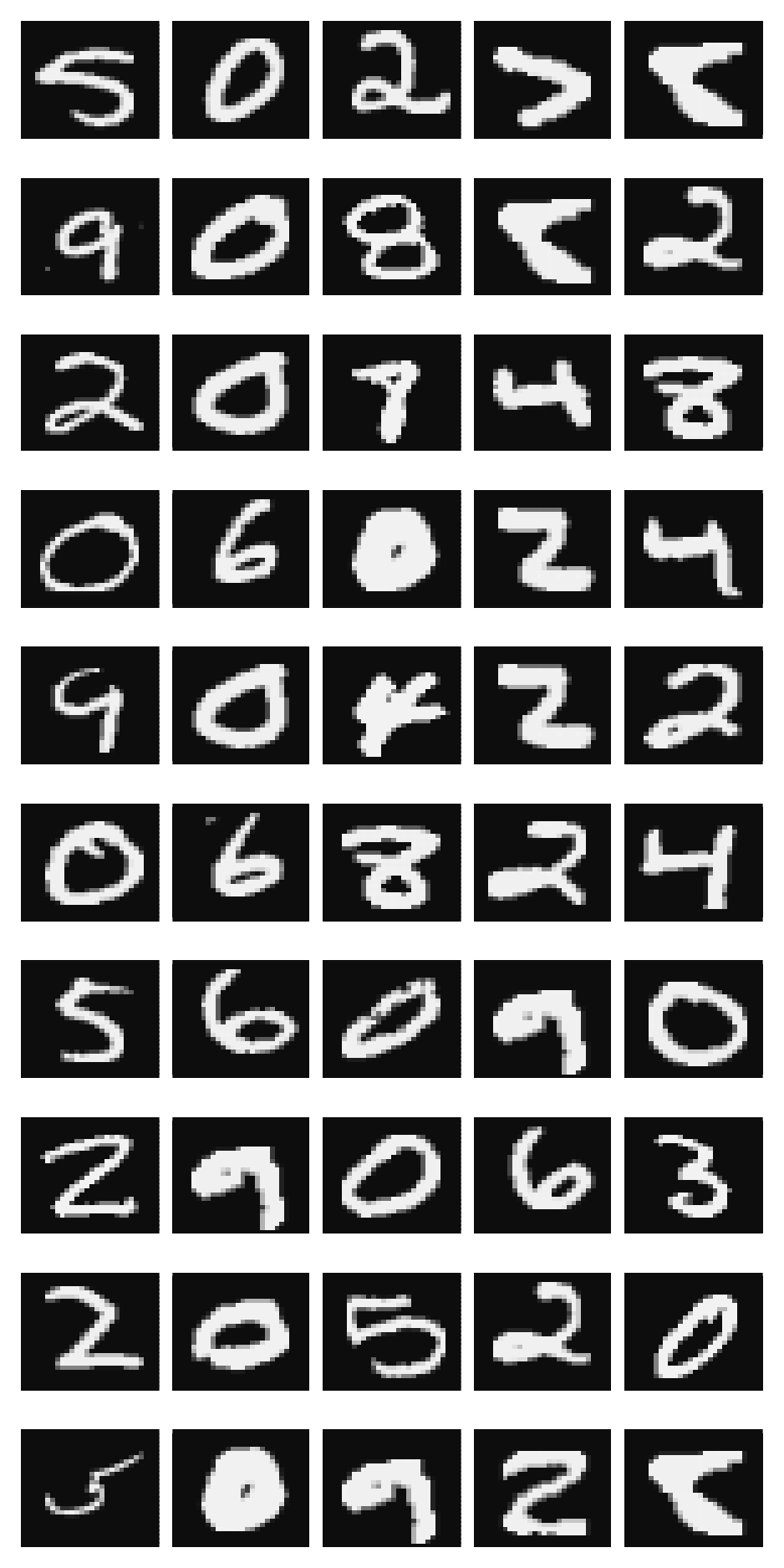}
        \subcaption{$\lambda = 0, \beta = 50$}
        \label{subfig:lambda_0_beta_50}
    \end{minipage}
    \begin{minipage}{0.32\columnwidth}
        \centering
        \includegraphics[width = 0.8\columnwidth]{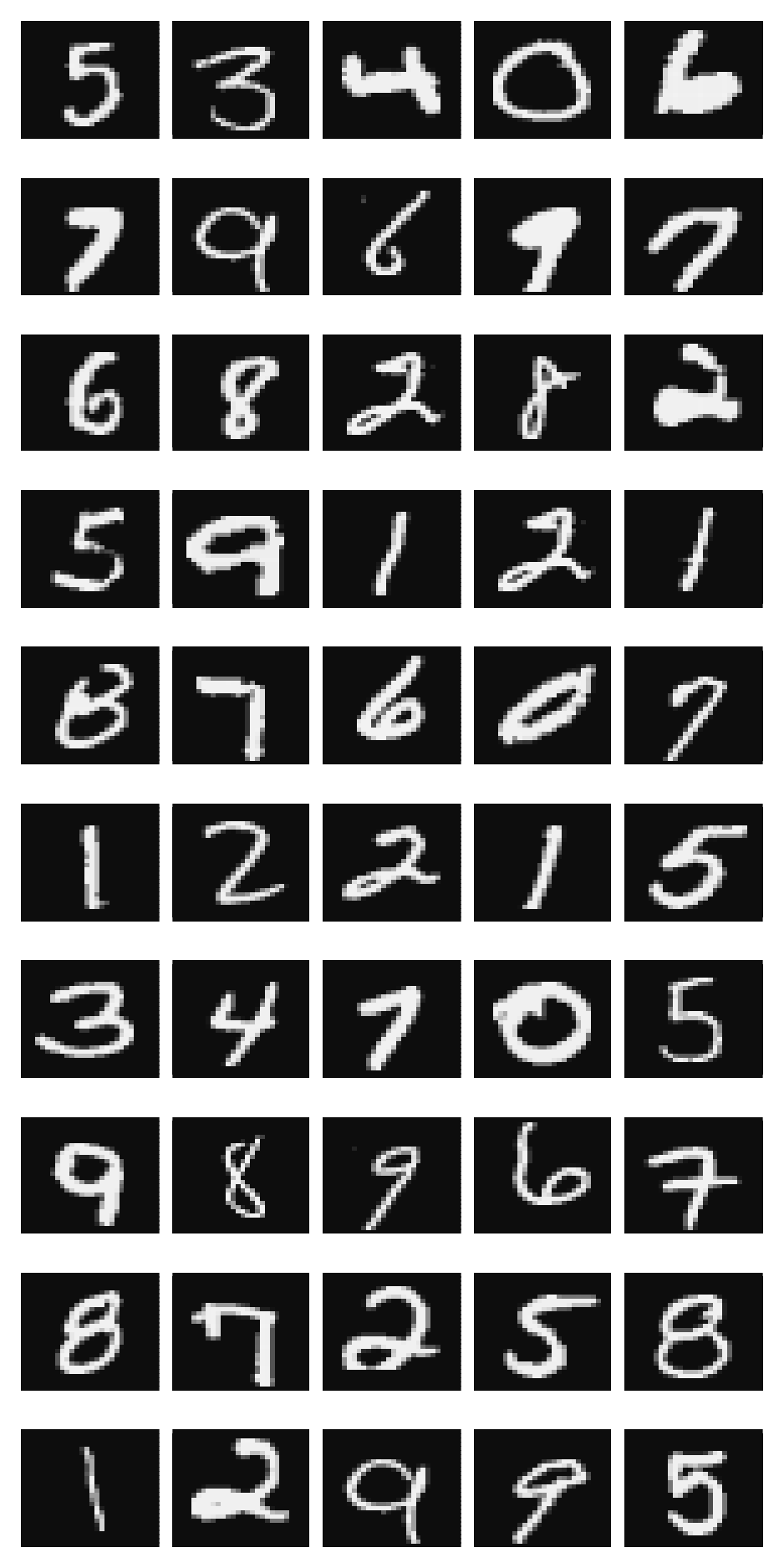}
        \subcaption{$\lambda = 1, \beta = 50$}
        \label{subfig:lambda_1_beta_50}
    \end{minipage}
    \begin{minipage}{0.32\columnwidth}
        \centering
        \includegraphics[width = 0.8\columnwidth]{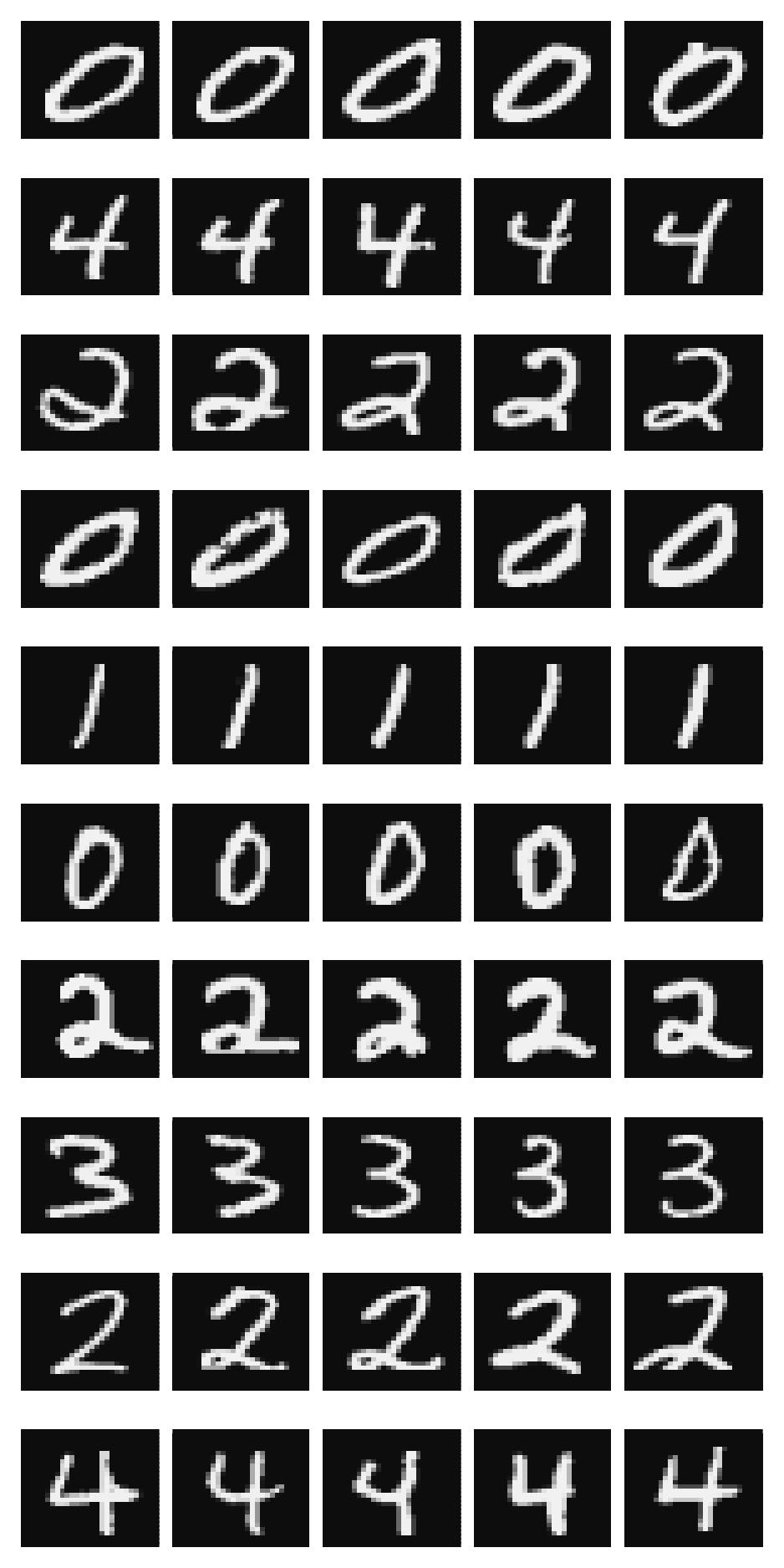}
        \subcaption{$\lambda = 2, \beta = 50$}
        \label{subfig:lambda_2_beta_50}
    \end{minipage}
    \caption{Examples of the acquired subsets of MNIST for $\beta = 50$.}
    \label{fig:beta_50}
\end{figure}

\end{document}